\documentclass{article}

\usepackage{arxiv}

\usepackage{natbib}

\usepackage{microtype}
\usepackage{graphicx}
\usepackage{subcaption}
\usepackage{booktabs} 
\usepackage{placeins}

\usepackage{hyperref}

\usepackage{algorithm}
\usepackage{algorithmic}

\usepackage{hyperref}
\usepackage{url}

\usepackage{amsmath}
\usepackage{amssymb}
\usepackage{mathtools}
\usepackage{amsthm}

\usepackage{algorithm}
\usepackage{algorithmic}

\usepackage{multirow}

\theoremstyle{plain}
\newtheorem{theorem}{Theorem}[section]
\newtheorem{proposition}[theorem]{Proposition}
\newtheorem{lemma}[theorem]{Lemma}

\theoremstyle{definition}

\title{Learning to Shuffle: Block Reshuffling and Reversal Schemes for Stochastic Optimization}

\author{Lam M. Nguyen$^{1}$, Dzung T. Phan$^{1}$, Jayant Kalagnanam$^{1}$ \\
$^{1}$ IBM Research, Thomas J. Watson Research Center, Yorktown Heights, NY, USA \\
\\
\texttt{LamNguyen.MLTD@ibm.com}, \texttt{phandu@us.ibm.com},
\texttt{jayant@us.ibm.com}
}

\usepackage{pifont}

\usepackage{epsfig}
\usepackage{amssymb}
\usepackage{amsmath}
\usepackage{amsthm}
\usepackage{amsfonts}
\usepackage{bbding}
\usepackage{array}
\usepackage{paralist}
\usepackage{xargs}                      
\usepackage{caption}

\hypersetup{
  colorlinks   = true, 
  urlcolor     = blue, 
  linkcolor    = blue, 
  citecolor   = blue 
}

\newcolumntype{C}[1]{>{\centering\let\newline\\\arraybackslash\hspace{0pt}}m{#1}}

\newcommand\tagthis{\addtocounter{equation}{1}\tag{\theequation}}

\usepackage{xcolor}

\newcommand{\zero}[1]{{\boldsymbol{0}}}

\usepackage{multirow}


\begin{document}

\maketitle

\begin{abstract}
Shuffling strategies for stochastic gradient descent (SGD), including incremental gradient, shuffle-once, and random reshuffling, are supported by rigorous convergence analyses for arbitrary within-epoch permutations. In particular, random reshuffling is known to improve optimization constants relative to cyclic and shuffle-once schemes. However, existing theory offers limited guidance on how to design new data-ordering schemes that further improve optimization constants or stability beyond random reshuffling. In this paper, we design a pipeline using a large language model (LLM)-guided program evolution framework to discover an effective shuffling rule for without-replacement SGD. Abstracting from this instance, we identify two fundamental structural components: block reshuffling and paired reversal. We analyze these components separately and show that block reshuffling strictly reduces prefix-gradient variance constants within the unified shuffling framework, yielding provable improvements over random reshuffling under mild conditions. Separately, we show that paired reversal symmetrizes the epoch map and cancels the leading order-dependent second-order term, reducing order sensitivity from quadratic to cubic in the step size. Numerical experiments with the discovered algorithm validate the theory and demonstrate consistent gains over standard shuffling schemes across convex and nonconvex benchmarks.
\end{abstract}

\section{Introduction}

Stochastic gradient descent (SGD) remains a fundamental optimization method for large-scale machine learning \citep{Nemirovski2009,ghadimi2013stochastic,Bottou2018}.
In practice, SGD is almost always implemented in an epoch-based, without replacement manner, where the training data are processed according to a permutation at the beginning of each epoch.
Such shuffling strategies are standard in modern machine learning frameworks, including TensorFlow, PyTorch, and Keras \citep{tensorflow2015-whitepaper,pytorch,chollet2015keras}.
Classical approaches include incremental gradient, shuffle-once, and random reshuffling, whose convergence properties are now well understood under standard smoothness assumptions \citep{nedic2001incremental}. Incremental gradient (IG) uses a fixed data order, shuffle-once (SO) samples one random order and reuses it across epochs, and random reshuffling (RR) draws a new random permutation at each epoch. In particular, recent work shows that RR often enjoys improved optimization constants compared to IG and SO \citep{mishchenko2020random}, and unified shuffling analyses establish convergence guarantees for arbitrary within-epoch permutations \citep{nguyen2021unified}.

Despite these advances, existing theory largely treats permutations as unstructured objects.
As a result, it provides limited guidance on how to design new data-ordering schemes that further improve optimization constants or enhance stability beyond RR. In this work, we adopt a different perspective.
Rather than proposing a shuffling rule solely from analytical intuition, we construct a pipeline based on large language model (LLM)-guided program evolution to explore the space of valid without replacement shuffling rules.
This process yields a concrete adaptive reshuffling algorithm, which we formalize as \emph{Adaptive Block Reshuffling with Periodic Transforms} (APR) in Algorithm~\ref{alg:apr}.
APR consistently outperforms standard shuffling schemes across a range of benchmarks.
The LLM is used only offline as a discovery tool and does not participate in training or inference.

Motivated by the structure of the discovered algorithm, we abstract away from its specific parameters and isolate two recurring components: \emph{block reshuffling} and \emph{paired reversal}.
We analyze these components separately.
Within the unified shuffling framework, we show that block reshuffling strictly reduces prefix-gradient variance constants under mild conditions, yielding provable improvements over the existing shuffling schemes.
Separately, we show that paired reversal symmetrizes the epoch map and cancels the leading order-dependent second-order term, reducing order sensitivity from quadratic to cubic in the learning rate.

Our experiments, conducted using the LLM-discovered APR algorithm, validate these theoretical predictions and demonstrate consistent improvements over IG, SO, and RR on both convex and non-convex problems.
Together, these results suggest that structured data-ordering schemes beyond random reshuffling can offer meaningful optimization benefits, and that automated discovery can help reveal such structure.

\paragraph{Contributions.}
Our main contributions are:
\begin{itemize}
    \item We introduce an LLM-guided program evolution pipeline to discover effective shuffling rules for without-replacement SGD, resulting in the APR algorithm.
    \item We abstract the discovered algorithm into two core components, \textit{block reshuffling} and \textit{paired reversal}, and study them independently. To the best of our knowledge, these components have not been previously formalized as independent structures within a unified theory for without replacement SGD. 

    \item We show that block reshuffling reduces prefix-gradient variance constants within the unified shuffling framework, yielding provable improvements over existing schemes.
    \item We prove that paired reversal cancels the leading order-dependent second-order term in the epoch map, reducing order sensitivity from $\Theta(\gamma^2)$ to $O(\gamma^3)$ (in learning rate $\gamma$).
    \item We demonstrate empirically that APR improves performance over standard shuffling schemes across convex and nonconvex benchmarks.
\end{itemize}

\subsection{Related Work}

\paragraph{Without-replacement SGD and shuffling schemes.}
Without-replacement SGD is widely used in finite-sum optimization.
Classical shuffling strategies include IG, SO, and RR.
Their convergence properties have been extensively analyzed in both convex and nonconvex settings
\citep{Gurbuzbalaban2019,haochen2019random,safran2020good,nagaraj2019sgd,rajput2020closing,nguyen2021unified,mishchenko2020random,ahn2020sgd,pmlr-v139-tran21b,Tran2022_ShufflingNesterov,nguyen2023convergence},
with RR often exhibiting improved optimization constants relative to IG and SO.
Unified shuffling analyses further establish convergence guarantees for arbitrary within-epoch permutations.
However, these results largely treat the permutation as an unstructured object and do not provide constructive principles for designing improved data-ordering schemes beyond RR.

\paragraph{Variance and order sensitivity reductions.}
Variance reduction plays a central role in improving the efficiency and stability of stochastic optimization.
Classical variance reduction methods such as SAG \citep{SAG}, SAGA \citep{SAGA}, SVRG \citep{SVRG}, and SARAH \citep{Nguyen2017b} achieve this by storing gradient information or computing periodic full gradients.
Shuffling based methods do not reduce variance in the same sense, but the choice of data ordering affects the variance constants appearing in without replacement analyses.
Our results show that block reshuffling reduces prefix-gradient variance constants relative to existing shuffling schemes without requiring additional gradient evaluations or storage.

Paired reversal addresses a complementary issue: sensitivity to the order of data within each epoch.
By symmetrizing the epoch update through averaging a permutation and its reverse, paired reversal cancels the leading order-dependent second-order term in the epoch-map expansion.
This reduces permutation induced variability and leads to more stable optimization dynamics.

\paragraph{LLM-guided algorithm discovery.}
Recent work demonstrates that large language models can assist in discovering algorithms through program synthesis and evolution (see e.g. \citep{romera2024mathematical,novikov2025alphaevolve}).
In contrast to approaches that deploy LLMs during training or inference, we use an LLM purely as an offline discovery tool.
The resulting algorithm is explicit and deterministic, and our theoretical analysis applies independently of the LLM itself.

\section{Problem Setup and Background}

We consider the finite-sum optimization problem
\begin{equation}\label{eq_finite_sum_opt_problem}
\min_{w \in \mathbb{R}^d} F(w) := \frac{1}{n} \sum_{i=1}^n f_i(w).
\end{equation}
We study epoch-based without replacement stochastic gradient methods.
At the beginning of each epoch $e$, a permutation $\pi_e$ of $\{1,\dots,n\}$ is generated, and the iterates are updated sequentially as $w_{t} = w_{t-1} - \gamma \nabla f_{\pi_e(t)}(w_{t-1})$, $t=1,\dots,n$. Let $T_{\pi}$ denote the epoch map corresponding to permutation $\pi$, i.e., $T_{\pi}(w)$ is the iterate after one full epoch starting from $w$.
Recent analyses characterize convergence in terms of permutation sensitive quantities such as prefix-gradient sums and associated variance constants. These quantities will play a central role in our analysis of block reshuffling \citep{mishchenko2020random}. The unified analyses show that convergence guarantees hold for any permutation selected at each epoch \citep{nguyen2021unified}. Consequently, the relevant algorithmic question is no longer whether a permutation converges, but rather how the structure of the permutation affects optimization efficiency.

\subsection{LLM-Guided Discovery of Shuffling Rules}

We use OpenEvolve \citep{openevolve}, an LLM-guided program evolution framework, to discover effective reshuffling rules for without-replacement SGD. The search space was restricted to deterministic Python functions producing permutations of $\{1,\dots,n\}$ (see Appendix~\ref{sec_app_LLM_guided} for more detail). 
Guided by empirical training feedback such as loss statistics, the LLM proposes concrete shuffling programs with fixed parameter choices and was prohibited from modifying the model, optimizer, or learning rate. One such discovered instance consistently outperformed standard shuffling schemes in preliminary experiments. The detailed settings are in Appendix~\ref{sec_app_experiment_settings}.

Motivated by the LLM-discovered instance, we generalize its behavior into Algorithm~\ref{alg:apr}, which generates a single permutation $\pi_e \in \mathcal{S}_n$ at each epoch $e$, where $\mathcal{S}_n$ denotes the symmetric group on $n$ elements.
The permutation is selected based on the epoch-level training loss $\ell_e$ through the loss ratio $\rho_e = \ell_e / (\ell_{e-1} + \varepsilon)$, which determines the reshuffling regime.
Block sizes are specified by fractions $\alpha_{\mathrm{strong}}$ and $\alpha_{\mathrm{mild}}$, yielding
$b_{\mathrm{strong}} = \max(1, \lfloor \alpha_{\mathrm{strong}} n \rfloor)$ and
$b_{\mathrm{mild}} = \max(1, \lfloor \alpha_{\mathrm{mild}} n \rfloor)$.
The algorithm optionally applies simple sequence transforms, including reversal $\mathrm{Rev}(x_1,\dots,x_n) = (x_n,\dots,x_1)$ and even-odd interleaving $\mathrm{EO}(x) = (x_1,x_3,\dots;\,x_2,x_4,\dots)$.
All randomness is governed by a deterministic epoch dependent seed $u_e$.
Importantly, Algorithm~1 always outputs a valid permutation and therefore remains within the standard without-replacement SGD framework, making it compatible with unified shuffling convergence analyses.
A detailed description of this procedure is provided in Appendix~\ref{sec_app_experiment_settings}.

\begin{algorithm}[t]
\caption{Adaptive Block Reshuffling with Periodic Transforms (APR)}
\label{alg:apr}
\begin{algorithmic}[1]
\REQUIRE Dataset size $n$, epoch index $e$, current loss $\ell_e$
\STATE \textbf{Parameters:} $0<\tau_{\mathrm{strong}}<\tau_{\mathrm{mild}}$, 
$\alpha_{\mathrm{strong}},\alpha_{\mathrm{mild}}\in(0,1]$,
$(p_{\mathrm{rev}},r_{\mathrm{rev}})$, $(p_{\mathrm{eo}},r_{\mathrm{eo}})$, $\varepsilon>0$
\STATE $u_e \leftarrow \mathrm{Seed}(e)$; \ $I \leftarrow (1,2,\dots,n)$
\STATE $b_{\mathrm{strong}} \leftarrow \max\!\bigl(1,\lfloor \alpha_{\mathrm{strong}} n \rfloor\bigr)$; \\
$b_{\mathrm{mild}} \leftarrow \max\!\bigl(1,\lfloor \alpha_{\mathrm{mild}} n \rfloor\bigr)$

\IF{$e=0$}
    \STATE Sample $\pi_0 \sim \mathrm{Unif}(\mathcal{S}_n)$ using seed $u_0$
    \STATE \textbf{output} $\pi_0$ and stop
\ENDIF

\STATE $\rho_e \leftarrow \ell_e/(\ell_{e-1}+\varepsilon)$

\IF{$\rho_e < \tau_{\mathrm{strong}}$}
    \STATE Partition $I$ into consecutive blocks of size $b_{\mathrm{strong}}$
    \STATE Randomly permute blocks using seed $u_e$; concatenate to form $\pi_e$
    \IF{$e \equiv r_{\mathrm{rev}} \;(\mathrm{mod}\; p_{\mathrm{rev}})$}
        \STATE $\pi_e \leftarrow \mathrm{Rev}(\pi_e)$
    \ENDIF
\ELSE
    \IF{$\rho_e < \tau_{\mathrm{mild}}$}
        \STATE Partition $I$ into consecutive blocks of size $b_{\mathrm{mild}}$
        \STATE Randomly permute blocks using seed $u_e$; concatenate to form $\pi_e$
    \ELSE
        \STATE Sample $\pi_e \sim \mathrm{Unif}(\mathcal{S}_n)$ using seed $u_e$
        \IF{$e \equiv r_{\mathrm{eo}} \;(\mathrm{mod}\; p_{\mathrm{eo}})$}
            \STATE $\pi_e \leftarrow \mathrm{EO}(\pi_e)$
        \ENDIF
    \ENDIF
\ENDIF

\STATE \textbf{output} $\pi_e$
\end{algorithmic}
\end{algorithm}

Inspection of the LLM-generated code and its generalized form reveals two recurring structural patterns: reshuffling data in contiguous blocks and reversing the data order to mitigate systematic ordering effects. Importantly, Algorithm~\ref{alg:apr} generates a single valid permutation at each epoch, and therefore remains fully within the standard without replacement SGD framework. As a result, its convergence behavior can be analyzed using existing unified shuffling theory, which applies to arbitrary (possibly data dependent) permutations. Rather than analyzing the full adaptive logic of Algorithm~\ref{alg:apr}, we isolate these two components for theoretical study. Block reshuffling partitions the dataset into consecutive blocks and randomly permutes the blocks while preserving the order within each block. This operation lies within the unified shuffling framework and is shown to strictly reduce prefix-gradient variance constants relative to standard schemes, including random reshuffling, under mild conditions. While the discovered instance applies one sided reversal heuristically, we analyze a paired reversal operator that averages the effects of a permutation and its reverse. This symmetrization cancels the leading order-dependent second-order term in the epoch map expansion, reducing order sensitivity from quadratic to cubic in the learning rate. 

\subsection{Block Reshuffling: Variance Reduction}

Block reshuffling partitions the index set $\{1,\dots,n\}$ into consecutive blocks of size $b$, randomly permutes the blocks, and processes the blocks sequentially while preserving the order within each block. Let $I = (1,2,\dots,n)$ and partition $I$ into blocks $\mathcal{B}(I;b) = (B_1,\dots,B_K)$ and $B_k = \bigl((k-1)b+1,\dots,\min(kb,n)\bigr)$,
where $K = \lceil n/b \rceil$.
Given a random permutation $\sigma$ of $\{1,\dots,K\}$, the block-reshuffled order is defined as $\mathrm{BlockShuffle}(I;b,\sigma) = (B_{\sigma(1)},\dots,B_{\sigma(K)})$. 

Block reshuffling interpolates between classical schemes: (1) $b=1$ corresponds to random reshuffling (RR); and (2) $b=n$ corresponds to incremental gradient (IG) methods.

From a statistical perspective, block reshuffling can be viewed as sampling block-averaged gradients without replacement rather than individual gradients. Let $G_k(w)$ denote the average gradient within block $k$. A standard variance decomposition yields $\sigma^2_{\mathrm{ind}}(w) =
\sigma^2_{\mathrm{within}}(w) + \sigma^2_{\mathrm{blk}}(w)$, where $\sigma^2_{\mathrm{blk}}(w)$ is the variance of block-averaged gradients. Whenever gradients within blocks are non-identical, $\sigma^2_{\mathrm{blk}}(w) < \sigma^2_{\mathrm{ind}}(w)$, implying that block reshuffling strictly reduces the variance constants that govern without-replacement SGD. Importantly, this improvement is achieved without altering the convergence rate guarantees provided by unified shuffling theory.

Within the unified shuffling framework, convergence bounds depend on prefix-gradient variance constants that measure the deviation between partial sums of gradients along a permutation and the full gradient.
We show that block reshuffling strictly reduces these constants under mild coherence assumptions on the gradients within each block.

\subsection{Paired Reversal: Symmetrization of Order Effects}

Sequential gradient updates within an epoch introduce order-dependent drift relative to full gradient descent. Under standard smoothness assumptions, this drift appears at second order in the stepsize and depends explicitly on the permutation. The reversal operator $\mathrm{Rev}(\cdot)$ maps a permutation 
$\pi = (\pi(1),\dots,\pi(n))
\quad$ $\mapsto$ $\mathrm{Rev}(\pi) = (\pi(n),\dots,\pi(1))$. 

Instead of using a single permutation $\pi$, we consider the symmetrized epoch map obtained by averaging the updates induced by $\pi$ and $\mathrm{Rev}(\pi)$, that is, $\bar{T}_\pi = \tfrac{1}{2}(T_\pi + T_{\mathrm{rev}(\pi)})$. This construction preserves the without replacement structure while introducing a fundamental symmetry in the update sequence.

Under standard smoothness assumptions, the epoch map admits a second-order expansion in the stepsize. Averaging a permutation with its reversal cancels the entire order-dependent second-order term, reducing permutation sensitivity from $\Theta(\gamma^2)$ to $O(\gamma^3)$. As a consequence, paired reversal provably stabilizes without replacement SGD beyond what is achievable by random reshuffling alone. Importantly, this symmetrization does not introduce additional randomness. It purely exploits structural symmetry in the permutation.


\section{Advantages of Structured Shuffling}
\label{sec:advantages-main}

This section presents two complementary, provable mechanisms that improve the practical behavior of without-replacement SGD: (i) variance-constant reduction via block reshuffling, and (ii) order-sensitivity reduction via paired reversal. All proofs are deferred to Appendix~\ref{sec:appendix-proofs}.

\subsection{Variance Reduction by Block Reshuffling}
\label{sec:block-main}

Fix an iterate $w\in\mathbb{R}^d$. Define component gradients and the full gradient
\begin{equation}
g_i(w):=\nabla f_i(w), 
\qquad 
\nabla F(w):=\frac{1}{n}\sum_{i=1}^n g_i(w).
\end{equation}
Define the population variance of individual gradients
\begin{equation}
\sigma_{\mathrm{ind}}^2(w)
:=\frac{1}{n}\sum_{i=1}^n \|g_i(w)-\nabla F(w)\|^2.
\label{eq:def-sigma-ind-main}
\end{equation}

Partition $\{1,\dots,n\}$ into $K$ disjoint blocks $B_1,\dots,B_K$ of equal size $b=n/K$ (assume $K\mid n$ for simplicity).
Define the block-averaged gradients
\begin{equation}
G_r(w):=\frac{1}{b}\sum_{i\in B_r} g_i(w),\qquad r=1,\dots,K,
\label{eq:def-block-grad-main}
\end{equation}
the block-level variance
\begin{equation}
\sigma_{\mathrm{blk}}^2(w)
:=\frac{1}{K}\sum_{r=1}^K \|G_r(w)-\nabla F(w)\|^2,
\label{eq:def-sigma-blk-main}
\end{equation}
and the within-block variance
\begin{equation}
\sigma_{\mathrm{within}}^2(w)
:=\frac{1}{K}\sum_{r=1}^K\frac{1}{b}\sum_{i\in B_r}\|g_i(w)-G_r(w)\|^2.
\label{eq:def-sigma-within-main}
\end{equation}

\begin{lemma}
\label{lem:exact-decomp-main}
With equal block sizes $b=n/K$,
\begin{equation}
\sigma_{\mathrm{ind}}^2(w)=\sigma_{\mathrm{within}}^2(w)+\sigma_{\mathrm{blk}}^2(w).
\label{eq:exact-decomp-main}
\end{equation}
Consequently, $\sigma_{\mathrm{blk}}^2(w)\le \sigma_{\mathrm{ind}}^2(w)$, with strict inequality whenever $\sigma_{\mathrm{within}}^2(w)>0$.
\end{lemma}
\noindent\textbf{Proof.} Deferred to Appendix~\ref{app:proof-exact-decomp}. \qed

Unified analyses for without-replacement SGD often control the deviation of prefix averages
\begin{equation}
\Delta_{\pi}^{(k)}(w)
:=\frac{1}{k}\sum_{t=1}^k g_{\pi(t)}(w)-\nabla F(w),
\label{eq:def-prefix-dev-main}
\end{equation}
where $\pi$ denotes the within-epoch permutation. The next lemma gives the exact second moment of the prefix mean under uniform sampling without replacement, a key ingredient in RR-type bounds.

\begin{lemma}[\citep{mishchenko2020random}]
\label{lem:wor-main}
Let $X_1,\dots,X_m\in\mathbb{R}^d$ be fixed with $\bar X=\frac{1}{m}\sum_{i=1}^m X_i$ and
$\sigma^2=\frac{1}{m}\sum_{i=1}^m\|X_i-\bar X\|^2$.
Let $\pi\sim\mathrm{Unif}(\mathcal{S}_m)$ and define $\bar X_\pi^{(k)}=\frac{1}{k}\sum_{t=1}^k X_{\pi(t)}$.
Then
\begin{equation}
\mathbb{E}\big[\bar X_\pi^{(k)}\big]=\bar X,
\qquad
\mathbb{E}\big\|\bar X_\pi^{(k)}-\bar X\big\|^2=\frac{m-k}{k(m-1)}\,\sigma^2.
\label{eq:wor-main}
\end{equation}
\end{lemma}


\begin{proposition}[RR prefix variance (sample-level)]
\label{prop:rr-prefix-main}
Let $\pi\sim\mathrm{Unif}(\mathcal{S}_n)$ and $\widehat g_m(w)=\frac{1}{m}\sum_{t=1}^m g_{\pi(t)}(w)$.
Then
\begin{equation}
\mathbb{E}\big\|\widehat g_m(w)-\nabla F(w)\big\|^2
=
\frac{n-m}{m(n-1)}\,\sigma_{\mathrm{ind}}^2(w).
\label{eq:rr-prefix-main}
\end{equation}
\end{proposition}
\noindent\textbf{Proof.} Deferred to Appendix~\ref{app:proof-rr-prefix}. \qed

\begin{proposition}[Block reshuffling prefix variance (block-level)]
\label{prop:block-prefix-main}
Let $\sigma\sim\mathrm{Unif}(\mathcal{S}_K)$ and $\widehat G_k(w)=\frac{1}{k}\sum_{j=1}^k G_{\sigma(j)}(w)$.
Then
\begin{equation}
\mathbb{E}\big\|\widehat G_k(w)-\nabla F(w)\big\|^2
=
\frac{K-k}{k(K-1)}\,\sigma_{\mathrm{blk}}^2(w).
\label{eq:block-prefix-main}
\end{equation}
\end{proposition}
\noindent\textbf{Proof.} Deferred to Appendix~\ref{app:proof-block-prefix}. \qed

Equations~\eqref{eq:rr-prefix-main}-\eqref{eq:block-prefix-main} show that the same without replacement prefix-variance formula holds at the block level, but with variance constant $\sigma_{\mathrm{blk}}^2(w)$ in place of $\sigma_{\mathrm{ind}}^2(w)$. By Lemma~\ref{lem:exact-decomp-main}, this constant is strictly smaller whenever within-block gradients are not identical. Thus, in regimes where gradients are locally coherent, block reshuffling reduces the magnitude of prefix-gradient noise and improves the constants appearing in convergence bounds driven by $\mathbb{E}\|\Delta_\pi^{(k)}(w)\|^2$.

\subsection{Order-Sensitivity Reduction by Paired Reversal}
\label{sec:rev-main}

We now formalize the benefit of paired reversal. Consider one epoch of SGD with stepsize $\gamma>0$ under a permutation $\pi\in\mathcal{S}_n$, starting from $w\in\mathbb{R}^d$:
\begin{align*}
    & w_0=w, \\
    & w_t=w_{t-1}-\gamma\,\nabla f_{\pi(t)}(w_{t-1}),\quad t=1,\dots,n,  \\
    & T_\pi(w):=w_n. \tagthis \label{eq:def-epoch-map-main}
\end{align*}

Let $\mathrm{Rev}(\pi)$ denote the reversed permutation $\mathrm{Rev}(\pi)(t)=\pi(n+1-t)$.

\paragraph{Assumptions.}
Assume each $f_i$, $i = 1,\dots,n$, is twice continuously differentiable and satisfies:
\begin{align}
\|\nabla f_i(x)-\nabla f_i(y)\| &\le L\|x-y\|, \label{eq:A1-main}\\
\|\nabla f_i(x)\| &\le G, \label{eq:A2-main}\\
\|\nabla^2 f_i(x)-\nabla^2 f_i(y)\| &\le \rho\|x-y\|. \label{eq:A3-main}
\end{align}
Define $g_i:=\nabla f_i(w)$ and $H_i:=\nabla^2 f_i(w)$.

\begin{lemma}
\label{lem:epoch-expansion-main}
Assume \eqref{eq:A1-main}-\eqref{eq:A3-main} hold. For any $\pi\in\mathcal{S}_n$,
\begin{equation}
T_\pi(w)
=
w-\gamma\sum_{i=1}^n g_i
+\gamma^2\sum_{1\le s<t\le n} H_{\pi(t)}\,g_{\pi(s)}
+R_\pi(w),
\label{eq:epoch-expansion-main}
\end{equation}
where the remainder satisfies
\begin{equation}
\|R_\pi(w)\|\le C_{\mathrm{rem}}\,\gamma^3 n^3,
\qquad
C_{\mathrm{rem}}:=\frac{\rho G^2}{2}+\frac{L^2G}{2}.
\label{eq:rem-bound-main}
\end{equation}
\end{lemma}
\noindent\textbf{Proof.} Deferred to Appendix~\ref{app:proof-epoch-expansion}. \qed

\begin{theorem}
\label{thm:paired-rev-main}
Define the paired-reversal averaged epoch map
\begin{equation}
\bar T_\pi(w):=\frac{1}{2}\Big(T_\pi(w)+T_{\mathrm{Rev}(\pi)}(w)\Big).
\label{eq:def-Tbar-main}
\end{equation}
Under \eqref{eq:A1-main}-\eqref{eq:A3-main} assumptions, $\bar T_\pi(w)$ admits the expansion
\begin{equation}
\bar T_\pi(w)
=
w-\gamma\sum_{i=1}^n g_i
+\frac{\gamma^2}{2}\sum_{i\neq j} H_i\,g_j
+\bar R_\pi(w),
\label{eq:paired-rev-expansion-main}
\end{equation}
where the second-order term is independent of $\pi$ and $\|\bar R_\pi(w)\|\le C_{\mathrm{rem}}\gamma^3 n^3$.
Consequently, for any $\pi,\pi'\in\mathcal{S}_n$,
\begin{equation}
\|\bar T_\pi(w)-\bar T_{\pi'}(w)\|
\le 2C_{\mathrm{rem}}\,\gamma^3 n^3,
\label{eq:order-sensitivity-main}
\end{equation}
i.e., order sensitivity is $O(\gamma^3)$ rather than $\Theta(\gamma^2)$ in general.
\end{theorem}
\noindent\textbf{Proof.} Deferred to Appendix~\ref{app:proof-paired-rev}. \qed

Theorem~\ref{thm:paired-rev-main} shows that averaging a permutation with its reversal removes the entire order-dependent $\gamma^2$ term in the epoch map expansion. Thus, when paired reversal is used, the leading contribution of permutation induced variability is pushed to higher order, providing a rigorous stability mechanism complementary to block variance reduction.

\subsection{Order Sensitivity: A Rigorous $O(\gamma^2)$ Upper Bound}
\label{sec:order-upper-main}

A standard second-order expansion immediately implies that, without any symmetrization, the dependence of the epoch map on the within-epoch order is \emph{at most quadratic} in the stepsize.

\begin{theorem}
\label{thm:order-upper-main}
Assume \eqref{eq:A1-main}-\eqref{eq:A3-main} hold. Let $T_\pi$ be the epoch map defined in \eqref{eq:def-epoch-map-main}.
Then for any $w\in\mathbb{R}^d$ and any two permutations $\pi,\pi'\in\mathcal{S}_n$,
\begin{equation}
\|T_\pi(w)-T_{\pi'}(w)\|
\;\le\;
LG\,\gamma^2\,n(n-1)
\;+\;
2C_{\mathrm{rem}}\gamma^3 n^3,
\label{eq:order-upper-main}
\end{equation}
where $C_{\mathrm{rem}}=\frac{\rho G^2}{2}+\frac{L^2G}{2}$ is the constant from Lemma~\ref{lem:epoch-expansion-main}.
In particular, for fixed $n$ the order sensitivity is $O(\gamma^2)$ as $\gamma\to 0$.
\end{theorem}
\noindent\textbf{Proof.} Deferred to Appendix~\ref{app:proof-order-upper}. \qed

\begin{proposition}
\label{prop:order-sens-lb-main}
There exist smooth finite-sum problems and two permutations $\pi,\pi'$ such that for all $\gamma>0$,
\[
\|T_\pi(w)-T_{\pi'}(w)\|\ge c\,\gamma^2
\]
for a constant $c>0$.
\end{proposition}
\noindent\textbf{Proof.} Deferred to Appendix~\ref{app:proof-order-lb}. \qed

Theorem~\ref{thm:order-upper-main} shows that permutation effects enter at second order through the term
$B_\pi(w)=\sum_{1\le s<t\le n} H_{\pi(t)}g_{\pi(s)}$ (Lemma~\ref{lem:epoch-expansion-main}),
so one cannot expect order dependence larger than $O(\gamma^2)$ for sufficiently small learning rate. Our explicit $n=2$ construction (Proposition~\ref{prop:order-sens-lb-main}) shows that $\Theta(\gamma^2)$ order dependence can occur, hence the quadratic scaling is unavoidable in general. Moreover, paired reversal (Theorem~\ref{thm:paired-rev-main}) cancels the entire order-dependent $\gamma^2$ term exactly, improving the bound from $O(\gamma^2)$ to $O(\gamma^3)$, which is a strict stability gain.

\subsection{Permutation-Variance Reduction and Sharpness}
\label{sec:permvar-main}

Paired reversal also reduces the permutation-induced variance of epoch maps.

\paragraph{Permutation variance.}
Let $\pi$ be random (typically $\pi\sim\mathrm{Unif}(\mathcal{S}_n)$). Define
\begin{align*}
    & \mathrm{Var}_\pi\!\big(T_\pi(w)\big)
:=\mathbb{E}_\pi\big\|T_\pi(w)-\mathbb{E}_\pi[T_\pi(w)]\big\|^2, \\
& \mathrm{Var}_\pi\!\big(\bar T_\pi(w)\big)
:=\mathbb{E}_\pi\big\|\bar T_\pi(w)-\mathbb{E}_\pi[\bar T_\pi(w)]\big\|^2. \tagthis \label{eq:def-perm-var-main}
\end{align*}

\paragraph{Second-order decomposition.}
From Lemma~\ref{lem:epoch-expansion-main},
\allowdisplaybreaks
\begin{align*}
    & T_\pi(w)=A(w)+\gamma^2 B_\pi(w)+R_\pi(w), \\
    & A(w):=w-\gamma\sum_{i=1}^n g_i, \tagthis \label{eq:T-decomp-main} \\
    & B_\pi(w):=\sum_{1\le s<t\le n} H_{\pi(t)}\,g_{\pi(s)}. 
\end{align*}

From Theorem~\ref{thm:paired-rev-main},
\begin{align*}
    & \bar T_\pi(w)=\bar A(w)+\bar R_\pi(w), \\
    & \bar A(w):=
w-\gamma\sum_{i=1}^n g_i
+\frac{\gamma^2}{2}\sum_{i\neq j} H_i\,g_j, \tagthis \label{eq:Tbar-decomp-main}
\end{align*}
where $\bar A(w)$ is deterministic (independent of $\pi$) and $\|\bar R_\pi(w)\|\le C_{\mathrm{rem}}\gamma^3 n^3$.

\begin{theorem}
\label{thm:perm-var-upper-main}
Assume \eqref{eq:A1-main}-\eqref{eq:A3-main} hold. Then for any fixed $w$,
\begin{align}
\mathrm{Var}_\pi\!\big(T_\pi(w)\big)
&\le
2\gamma^4\,\mathrm{Var}_\pi\!\big(B_\pi(w)\big)
+
2C_{\mathrm{rem}}^2\,\gamma^6 n^6,
\label{eq:var-T-upper-main}
\\
\mathrm{Var}_\pi\!\big(\bar T_\pi(w)\big)
&\le
C_{\mathrm{rem}}^2\,\gamma^6 n^6.
\label{eq:var-Tbar-upper-main}
\end{align}
\end{theorem}
\noindent\textbf{Proof.} Deferred to Appendix~\ref{app:proof-permvar-upper}. \qed

The following explicit constructions show that $\Theta(\gamma^2)$ order dependence of $T_\pi$ is unavoidable in general, and $\mathrm{Var}_\pi(T_\pi(w))=\Theta(\gamma^4)$ in general.

\begin{proposition}
\label{prop:permvar-lb-main}
There exist smooth finite-sum problems and an iterate $w$ such that for $\pi\sim\mathrm{Unif}(\mathcal{S}_2)$,
\[
\mathrm{Var}_\pi\!\big(T_\pi(w)\big)=\Theta(\gamma^4)\qquad (\text{for fixed }n=2).
\]
\end{proposition}
\noindent\textbf{Proof.} Deferred to Appendix~\ref{app:proof-permvar-lb}. \qed


We summarize the key properties of the two components as follows.
\begin{itemize}
\item \textbf{Block reshuffling.}
Lemma~\ref{lem:exact-decomp-main} shows that $\sigma_{\mathrm{blk}}^2(w)\le \sigma_{\mathrm{ind}}^2(w)$, with strict improvement when within-block gradients are non-identical.
Propositions~\ref{prop:rr-prefix-main}-\ref{prop:block-prefix-main} further show that block reshuffling obeys the same without-replacement prefix-variance formula with the reduced constant $\sigma_{\mathrm{blk}}^2(w)$, improving constants in unified bounds involving $\mathbb{E}\|\Delta_\pi^{(k)}(w)\|^2$.

\item \textbf{Paired reversal.}
Theorem~\ref{thm:paired-rev-main} shows that averaging a permutation and its reverse cancels the order-dependent $\gamma^2$ term, reducing order sensitivity from $\Theta(\gamma^2)$ to $O(\gamma^3)$.
\end{itemize}

Together, these results show that block reshuffling and paired reversal capture complementary mechanisms in without replacement SGD.
Block reshuffling reduces variance constants through structured aggregation, while paired reversal reduces order sensitivity through symmetry.
These principles suggest broader classes of structured permutations and open directions for designing new algorithms and convergence theory.

\section{Numerical Experiments}\label{sec_experiment}

We evaluate the empirical performance of the shuffling rule discovered by the LLM-guided program evolution pipeline (APR) and compare it against standard shuffling schemes on classification and regression tasks.

\begin{figure}[h]
    \centering

    \begin{subfigure}[t]{\textwidth}
        \centering
        \includegraphics[width=0.33\textwidth]{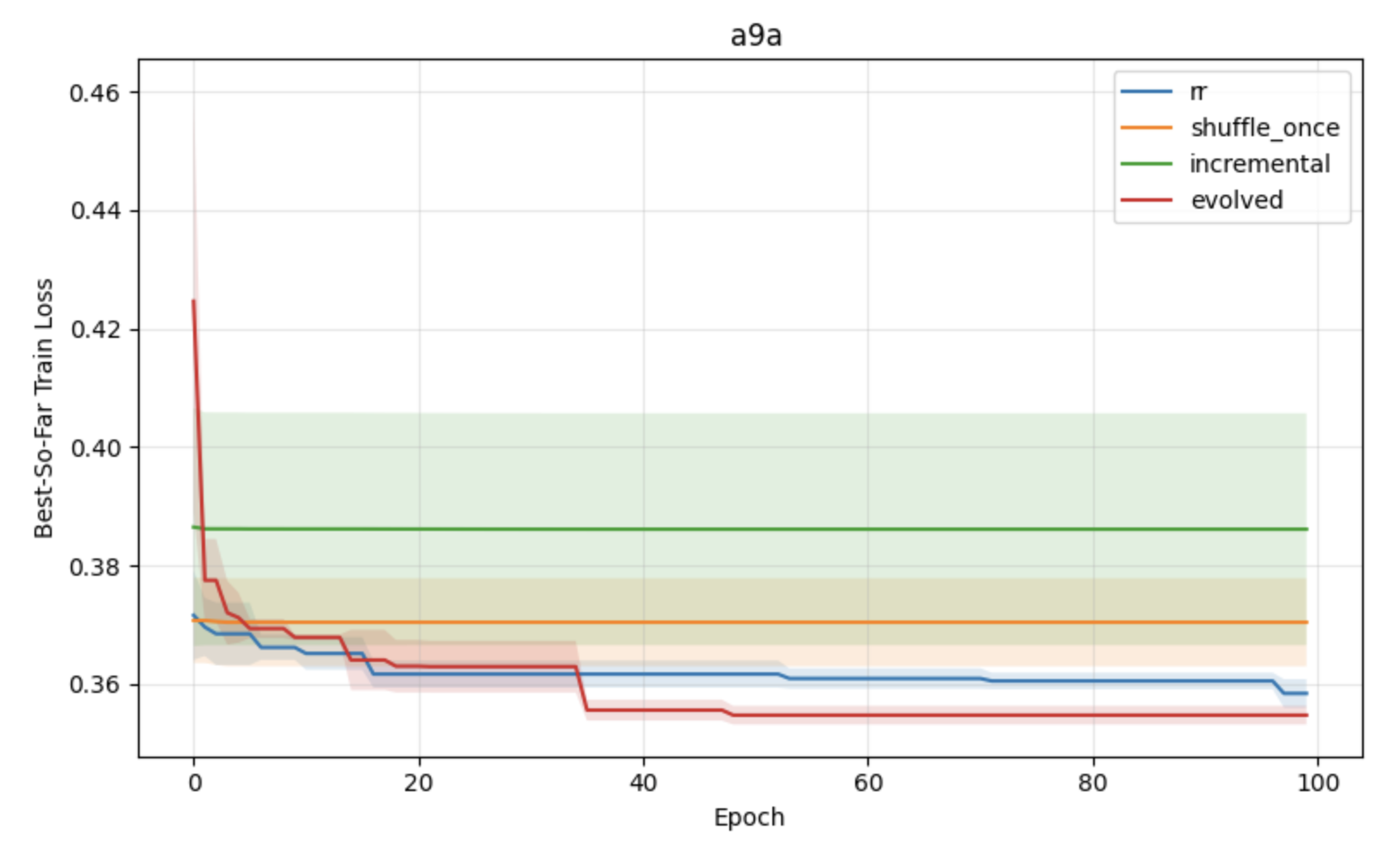}
        \includegraphics[width=0.33\textwidth]{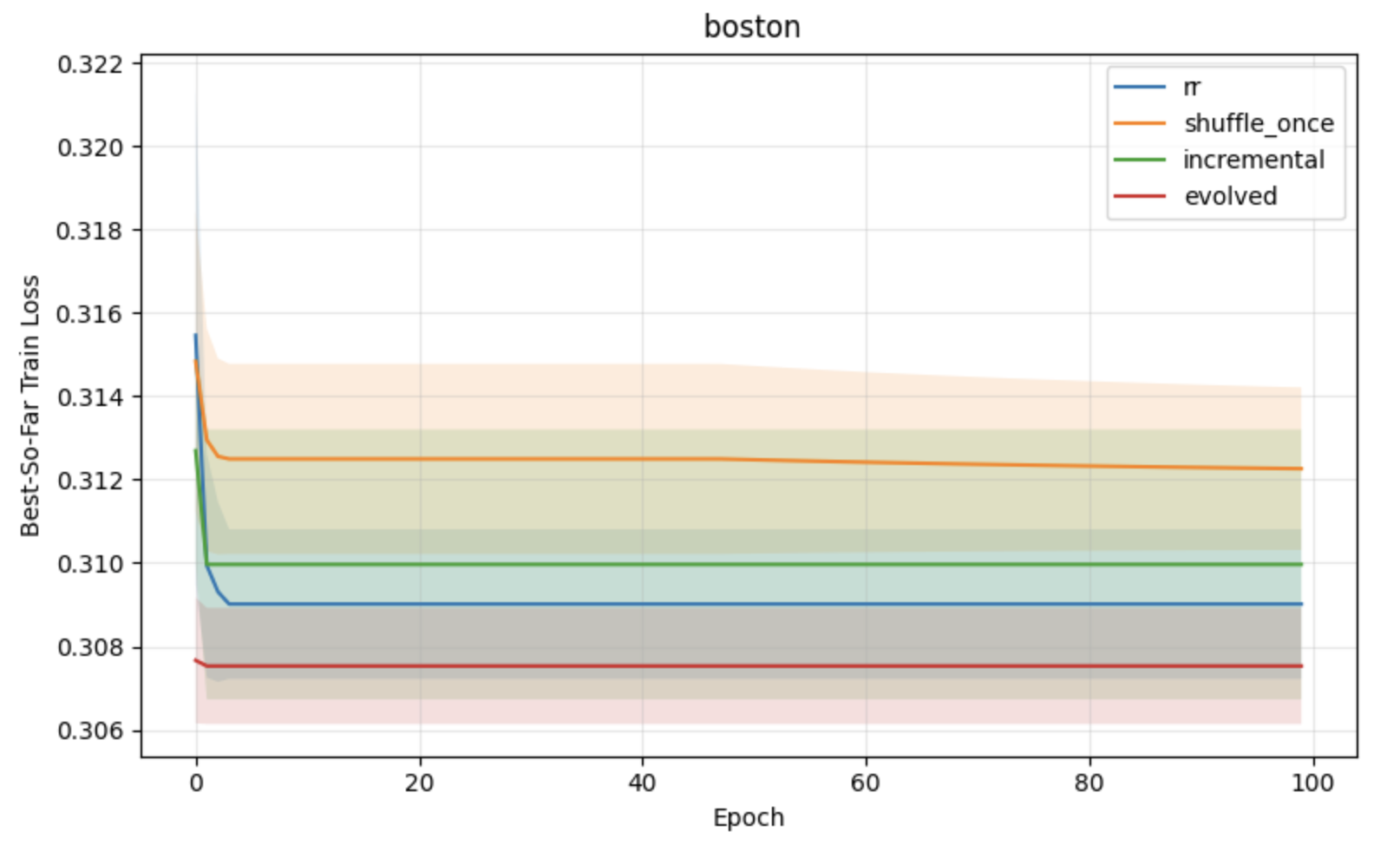}
        \includegraphics[width=0.33\textwidth]{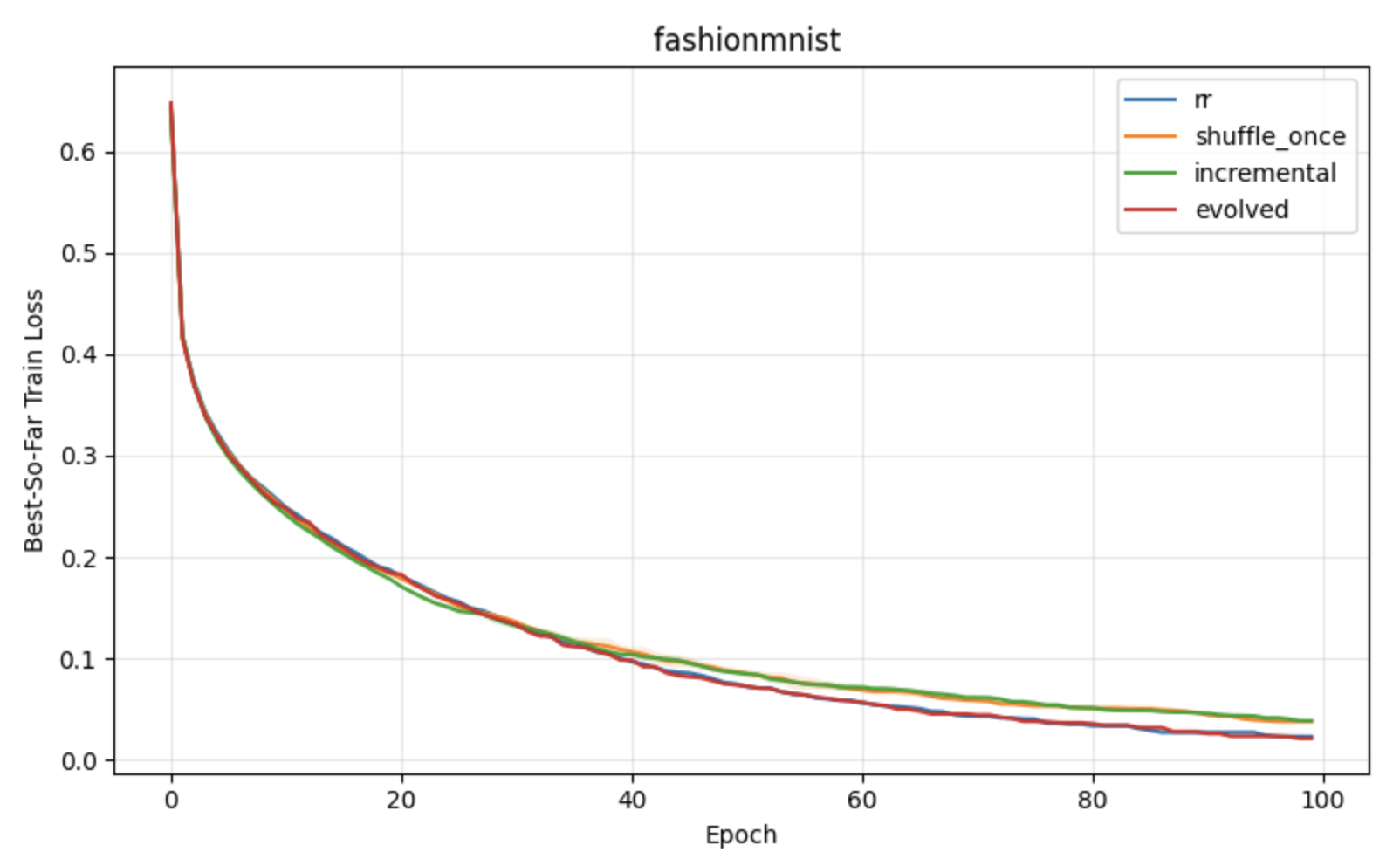}
        \caption{Constant learning rates}
        \label{fig:constant_lr_main}
    \end{subfigure}

    \vspace{0.2em}

    \begin{subfigure}[t]{\textwidth}
        \centering
        \includegraphics[width=0.33\textwidth]{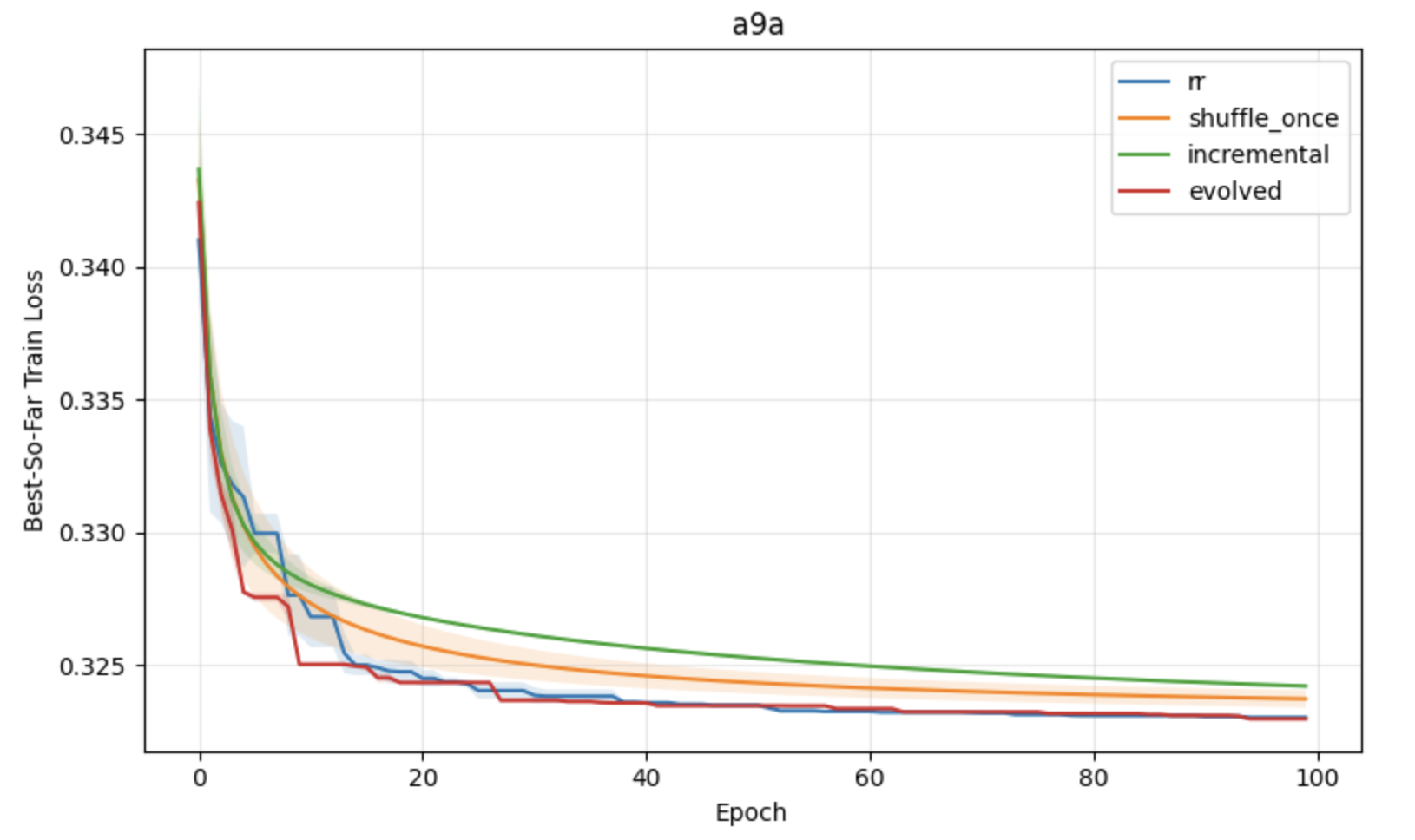}
        \includegraphics[width=0.33\textwidth]{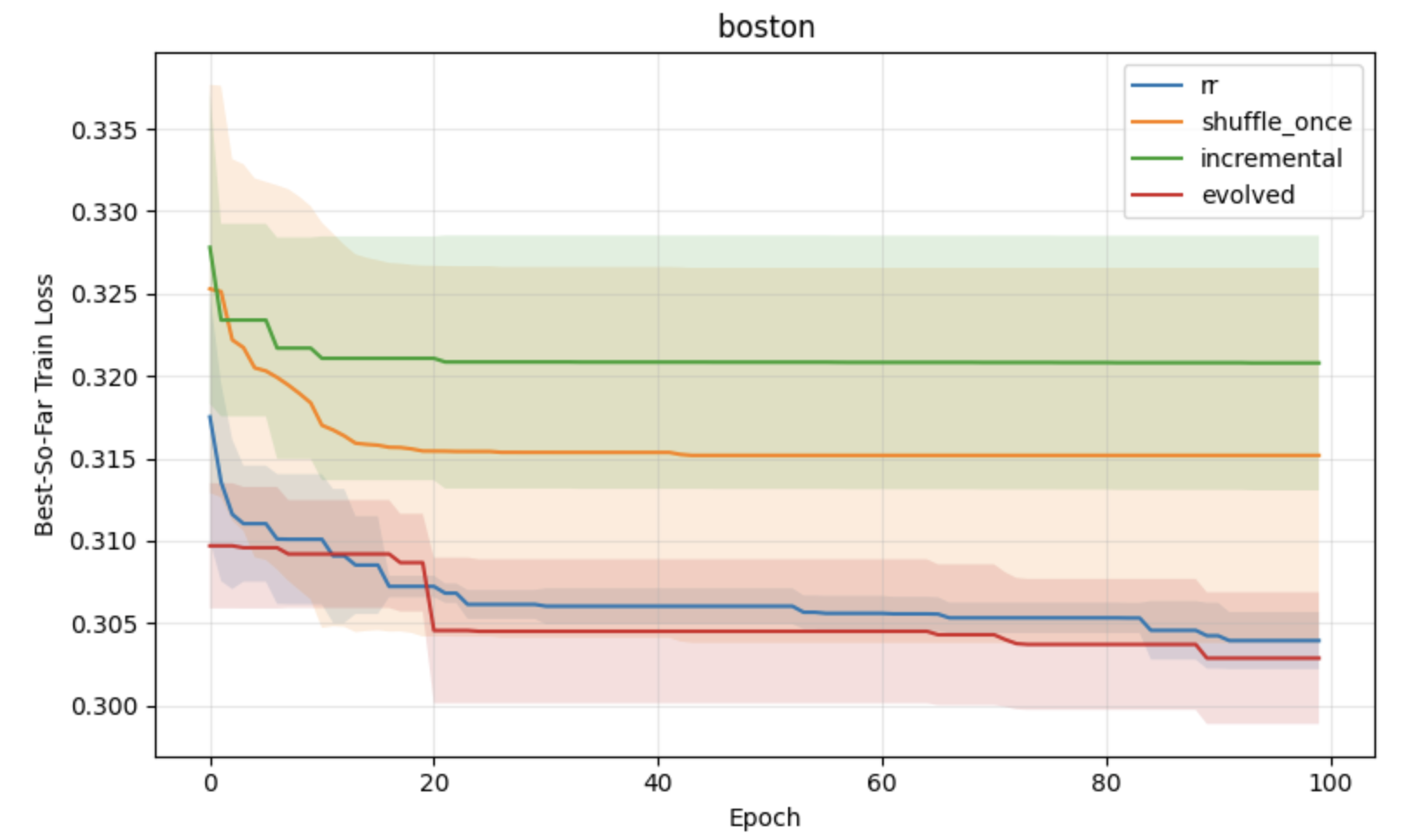}
        \includegraphics[width=0.33\textwidth]{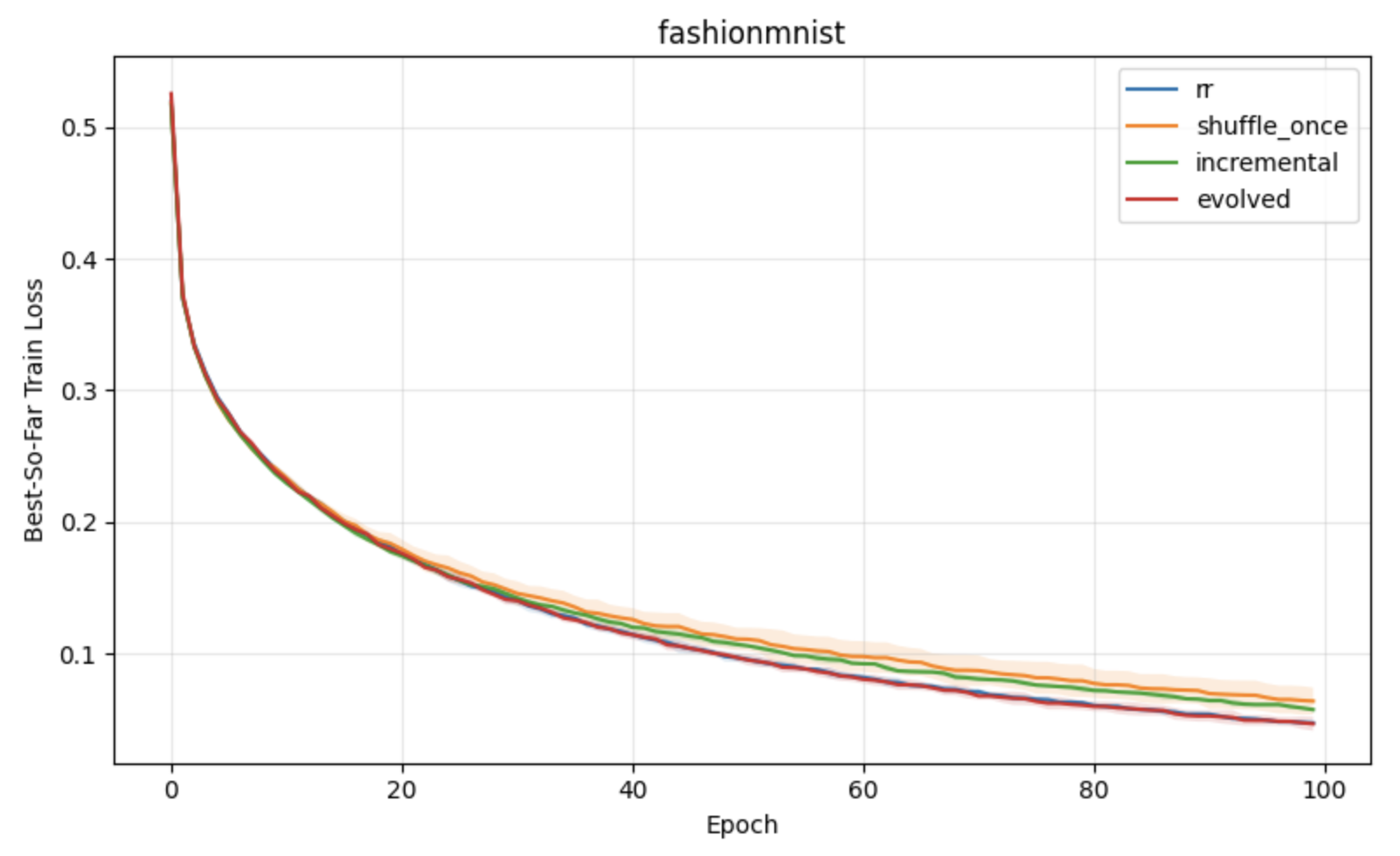}
        \caption{Diminishing learning rates}
        \label{fig:diminishing_lr_main}
    \end{subfigure}

    \vspace{0.2em}

    \begin{subfigure}[t]{\textwidth}
        \centering
        \includegraphics[width=0.33\textwidth]{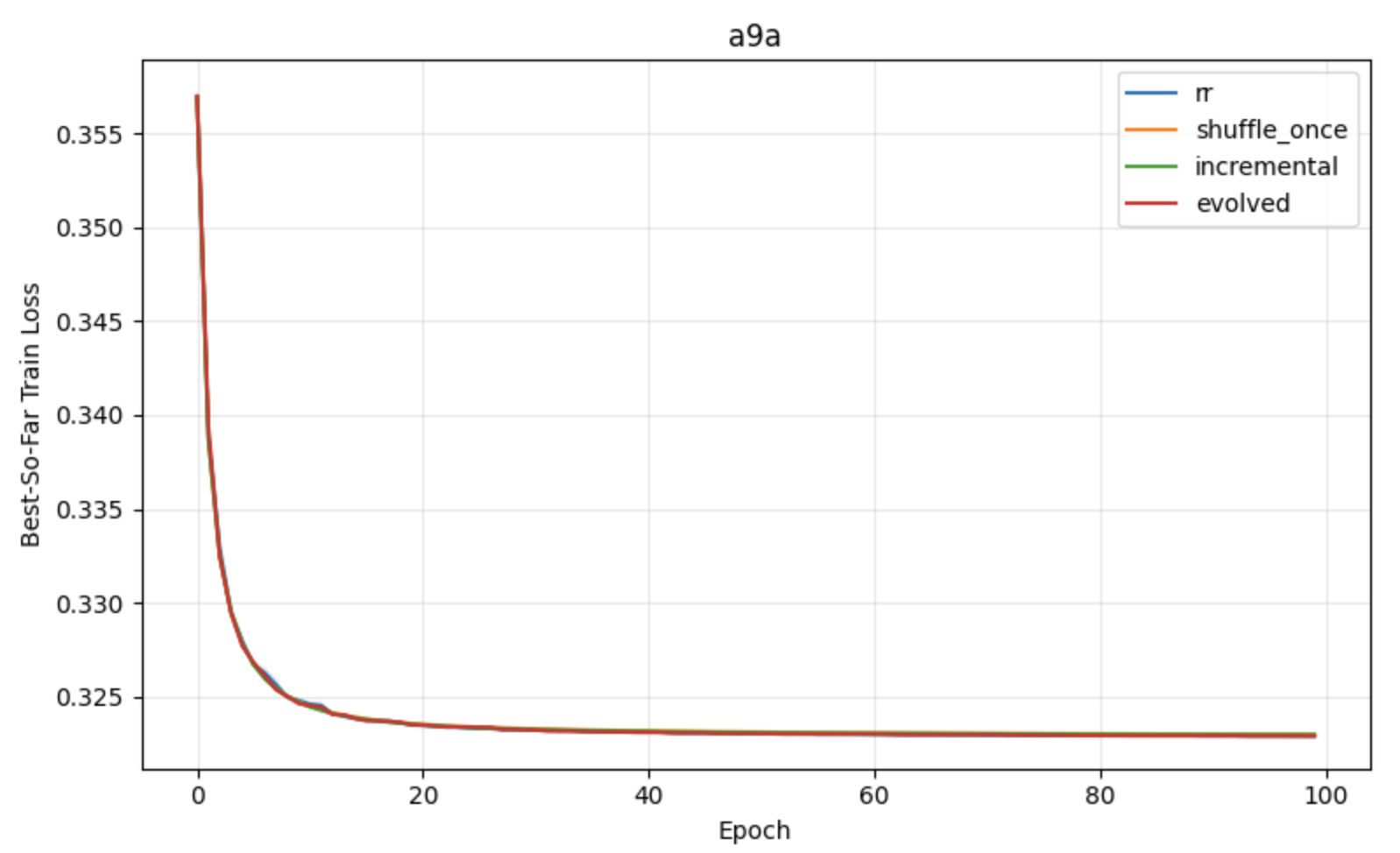}
        \includegraphics[width=0.33\textwidth]{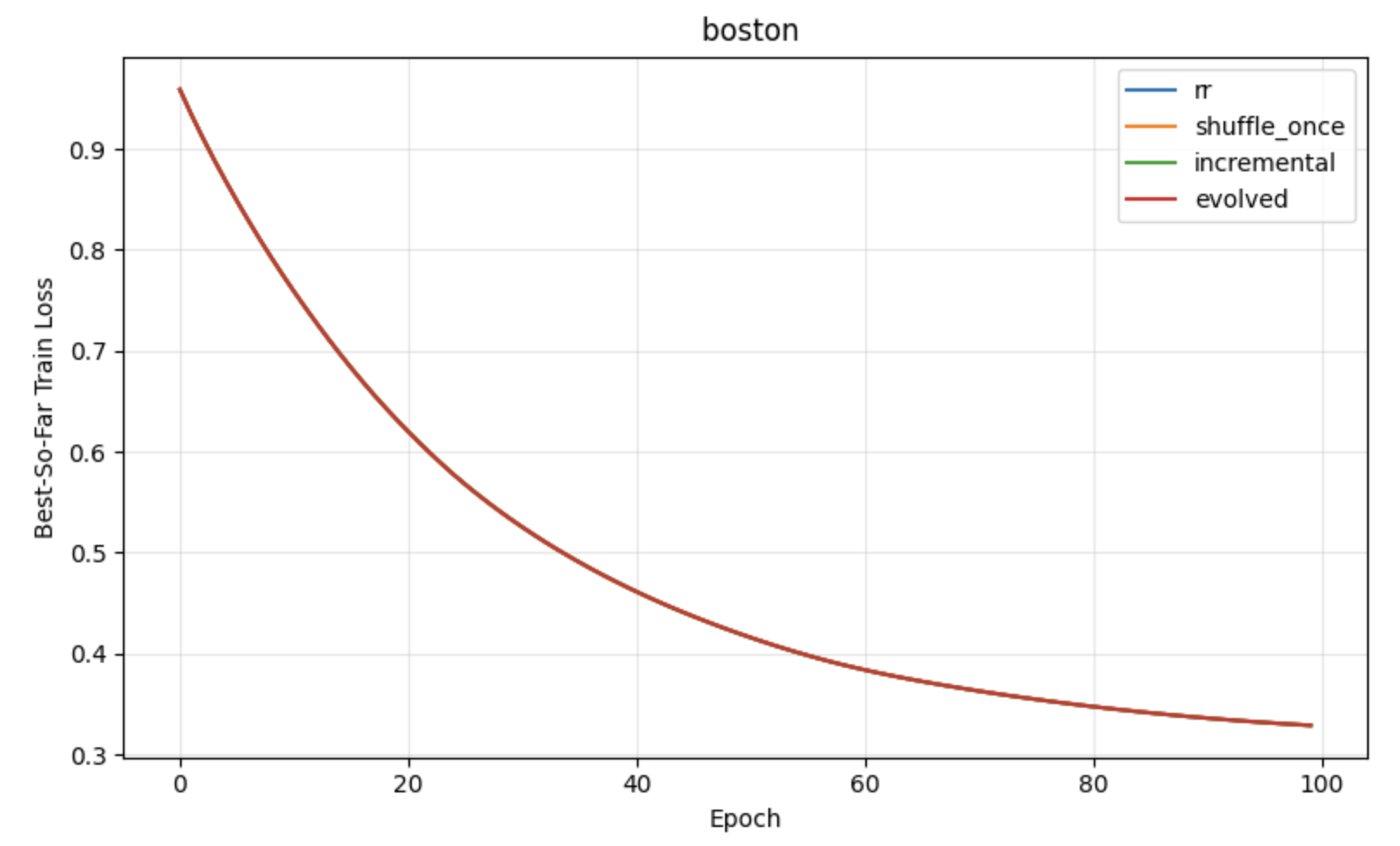}
        \includegraphics[width=0.33\textwidth]{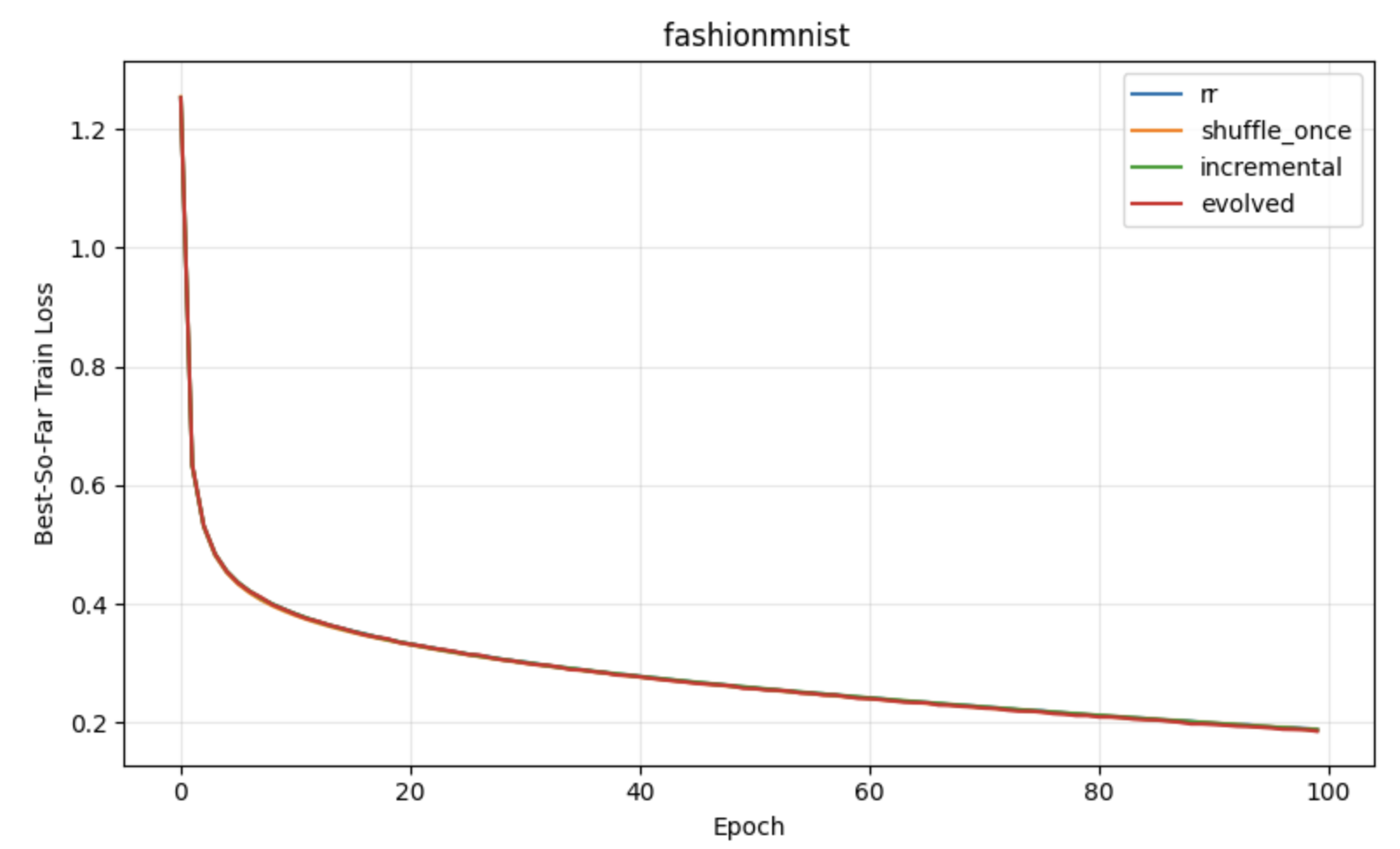}
        \caption{Small learning rates}
        \label{fig:small_variance_main}
    \end{subfigure}

    \caption{Performance comparison under different learning-rate schedules for \texttt{a9a}, \texttt{boston}, and \texttt{FashionMNIST} datasets.}
    \label{fig:main_experiments}
\end{figure}

For the \texttt{a9a}, \texttt{breast\_cancer}, and \texttt{digits} datasets \citep{LIBSVM,kaynak1995methods}, we use logistic regression models for classification, resulting in convex objectives. For the \texttt{california}, \texttt{boston}, and \texttt{diabetes} datasets \citep{pace1997sparse,buitinck2013api}, we use linear regression with mean squared error loss, yielding convex regression problems. For the \texttt{MNIST} and \texttt{FashionMNIST} datasets \citep{MNIST,xiao2017fashion}, we use fully connected neural networks for classification, leading to non-convex optimization problems. Additional details on datasets, models, and training configurations are provided in Appendix~\ref{sec_app_more_experiments}.

\begin{table*}[t]
\caption{Comparisons using different shuffling schemes to SGD. RR: Random Reshuffle; SO: Shuffle Once; IG: Incremental Gradient
}\label{table_1}
\begin{center}
\begin{tabular}{|l|l|l|l|l|l|l|} 
 \hline
 \textbf{Dataset} & \textbf{Task} & \textbf{LR} & \textbf{Evolved (APR)} & \textbf{RR} & \textbf{SO} & \textbf{IG} \\ 
 \hline
 \multirow{4}{*}{\textit{a9a}} & Classification & Fixed  &  \textcolor{blue}{$0.354733$}  & $0.358406$ & $0.370421$ & $0.386135$
  \\ 
     & (Convex) &  & \textcolor{red}{$\pm 0.001576$} & $\pm 0.002402$ & $\pm 0.007436$ & $\pm 0.019584$
  \\ 
  \cline{2-7}
     & Classification  & Dimishing  & \textcolor{blue}{$0.322985$} & $0.323040$ & $0.323738$ & $0.324219$
  \\ 
    &  (Convex) &  & \textcolor{red}{$\pm 0.000009$} & $\pm 0.000062$ & $\pm 0.000323$ & $\pm 0.000041$
 \\ 
  \hline
  \multirow{4}{*}{\textit{breast\_cancer}} & Classification & Fixed  & \textcolor{blue}{$0.043627$}   & $0.045695$ & $0.074279$ & $0.056839$
  \\ 
     & (Convex) &  & \textcolor{red}{$\pm 0.001694$} & $\pm 0.003468$ & $\pm 0.021865$ & $\pm 0.012038$
  \\ 
  \cline{2-7}
     & Classification & Dimishing  & $0.038790$ & \textcolor{blue}{$0.038632$} & $0.040084$ & $0.039313$
  \\ 
    & (Convex)  &  & $\pm 0.000249$ & \textcolor{red}{$\pm 0.000112$} & $\pm 0.001086$ & $\pm 0.000288$
 \\ 
  \hline
  \multirow{4}{*}{\textit{digits}} & Classification & Fixed  &  \textcolor{blue}{$0.199566$}  & $0.199883$  & $0.201699$ & $0.204401$
  \\ 
     & (Convex) &  & \textcolor{red}{$\pm 0.000046$} & $\pm 0.000162$ & $\pm 0.000904$ & $\pm 0.000331$
  \\ 
  \cline{2-7}
     & Classification & Dimishing  & \textcolor{blue}{$0.199568$} & $0.199974$ & $0.201680$ & $0.204488$
  \\ 
    & (Convex) &  & \textcolor{red}{$\pm 0.000048$} & $\pm 0.000230$  & $\pm 0.000910$ & $\pm 0.000327$
 \\ 
 \hline
 \multirow{4}{*}{\textit{california}} & Regression & Fixed  & \textcolor{blue}{$0.371610$}   & $0.375588$ & $0.460118$ & $0.568954$
  \\ 
     & (Convex) &  & \textcolor{red}{$\pm 0.002141$} & $\pm 0.004742$ & $\pm 0.029614$ & $\pm 0.027146$
  \\ 
  \cline{2-7}
     & Regression  & Dimishing  & \textcolor{blue}{$0.373325$} & $0.374319$  & $0.414826$ & $0.543131$
  \\ 
    & (Convex)  &  & \textcolor{red}{$\pm 0.001430$} & $\pm 0.002249$ & $\pm 0.019466$ & $\pm 0.020387$
 \\ 
  \hline
  \multirow{4}{*}{\textit{boston}} & Regression & Fixed  &  \textcolor{blue}{$0.307530$}  & $0.309014$ & $0.312260$ & $0.309964$
  \\ 
     & (Convex) &  & \textcolor{red}{$\pm 0.001394$} & $\pm 0.001794$ & $\pm 0.001948$ & $\pm 0.003237$
  \\ 
  \cline{2-7}
     & Regression  & Dimishing  & \textcolor{blue}{$0.308947$} & $0.309462$  & $0.313370$ & $0.313111$
  \\ 
    & (Convex)  &  & $\pm 0.003023$ & $\pm 0.001816$  & \textcolor{red}{$\pm 0.000803$} & $\pm 0.001262$
 \\ 
  \hline
  \multirow{4}{*}{\textit{diabetes}} & Regression & Fixed  &  \textcolor{blue}{$0.494990$}  & $0.495689$ & $0.508570$ & $0.497465$
  \\ 
     & (Convex) &  & $\pm 0.000777$ & \textcolor{red}{$\pm 0.000378$} & $\pm 0.008987$ & $\pm 0.002258$
  \\ 
  \cline{2-7}
     & Regression  & Dimishing  & \textcolor{blue}{$0.494849$} & $0.494986$ & $0.500723$ & $0.498138$
  \\ 
    & (Convex)  &  & \textcolor{red}{$\pm 0.000732$} & $\pm 0.000737$ & $\pm 0.000952$ & $\pm 0.002611$
 \\ 
 \hline
\end{tabular}
\end{center}
\end{table*}

\begin{table*}[!]
\caption{Comparisons using different shuffling schemes to Adam. RR: Random Reshuffle; SO: Shuffle Once; IG: Incremental Gradient
}\label{table_2}
\begin{center}
\begin{tabular}{|l|l|l|l|l|l|l|} 
 \hline
 \textbf{Dataset} & \textbf{Task} & \textbf{LR} & \textbf{Evolved (APR)} & \textbf{RR} & \textbf{SO} & \textbf{IG} \\ 
 \hline
  \multirow{4}{*}{\textit{MNIST}} & Classification & Fixed  &  \textcolor{blue}{$0.000001$}   & \textcolor{blue}{$0.000001$} & \textcolor{blue}{$0.000001$} & \textcolor{blue}{$0.000001$}
  \\ 
     & (Non-Convex) &  & \textcolor{red}{$\pm 0.000001$} & \textcolor{red}{$\pm 0.000001$} & \textcolor{red}{$\pm 0.000001$} & \textcolor{red}{$\pm 0.000001$}
  \\ 
  \cline{2-7}
     & Classification  & Dimishing  & $0.000016$   & $0.000024$ & \textcolor{blue}{$0.000003$} & $0.000005$
  \\ 
     & (Non-Convex) &  & \textcolor{red}{$\pm 0.000001$} & \textcolor{red}{$\pm 0.000001$} & \textcolor{red}{$\pm 0.000001$} & \textcolor{red}{$\pm 0.000001$}
 \\ 
  \hline
  \multirow{4}{*}{\textit{FashionMNIST}} & Classification & Fixed  &  \textcolor{blue}{$0.021357$}  & $0.023210$ & $0.038008$ & $0.038755$
  \\ 
     & (Non-Convex) &  & $\pm 0.002395$ & $\pm 0.000929$ & \textcolor{red}{$\pm 0.000675$} & $\pm 0.001205$
  \\ 
  \cline{2-7}
     & Classification  & Dimishing  & \textcolor{blue}{$0.022238$} & $0.022286$ & $0.040572$ & $0.046295$
  \\ 
    & (Non-Convex) &  & $\pm 0.003201$ & \textcolor{red}{$\pm 0.002430$} & $\pm 0.004566$ & $\pm 0.004722$
 \\ 
 \hline
\end{tabular}
\end{center}
\end{table*}

We compare the proposed shuffling strategy, the LLM-discovered APR scheme (Evolved), against standard baselines including incremental gradient (IG), shuffle-once (SO), and random reshuffling (RR) on a range of convex and nonconvex problems.
All experiments focus on training performance, consistent with our theoretical analysis.
For each dataset, model, learning-rate schedule, and shuffling scheme, we repeat experiments using $5$ random initializations and $5$ independent runs, and report the mean and standard deviation of the best-so-far training loss after 100 epochs.
Complete experimental settings are provided in Appendix~\ref{sec_app_experiment_settings}.

Table~\ref{table_1} reports results for SGD under constant and diminishing learning rates.
Across all convex classification and regression tasks, the evolved shuffling strategy consistently achieves lower best-so-far training loss than IG, SO, and RR.
The gains are most pronounced under constant learning rates, where order-dependent effects are strongest.
In addition to improved mean performance, the evolved scheme typically exhibits smaller standard deviations, indicating more stable optimization trajectories.
This behavior aligns with our theory: block reshuffling reduces prefix-gradient variance constants, leading to smoother within-epoch updates.
Under diminishing learning rates, performance differences become smaller, but the evolved scheme remains competitive and often retains a slight advantage. Table~\ref{table_2} summarizes results for Adam \citep{Kingma2014} on \texttt{MNIST} and \texttt{FashionMNIST}.
Although Adam introduces additional adaptivity beyond data ordering, the evolved shuffling strategy remains competitive and often improves best-so-far training loss relative to IG, SO, and RR.
These results suggest that structured shuffling can be beneficial beyond plain SGD.

Figure~\ref{fig:main_experiments} illustrates training dynamics on representative datasets (\texttt{a9a}, \texttt{boston}, and \texttt{FashionMNIST}) under different learning rate schedules. Under constant learning rates (Figure~\ref{fig:constant_lr_main}), the evolved scheme converges faster and to lower training loss, with visibly tighter confidence bands, reflecting reduced variance across runs. Under diminishing learning rates (Figure~\ref{fig:diminishing_lr_main} ), convergence curves become more similar, though the evolved scheme remains competitive throughout training. In the small-variance regime induced by very small learning rates (Figure~\ref{fig:small_variance_main}), all shuffling schemes exhibit nearly indistinguishable behavior. This empirical observation matches our theoretical results, which show that order sensitivity and permutation-induced variance vanish rapidly as the learning rate decreases. Consequently, differences between shuffling strategies become negligible in this regime.

Overall, the experiments support the theoretical insights of this work.
APR improves both training loss and stability, particularly when learning rates are not overly small.
These benefits are achieved without additional computation or memory overhead, highlighting the practical value of structured data ordering.

\section{Conclusion}\label{sec_conclusion}


We study structured data-ordering strategies for without-replacement stochastic optimization by combining
LLM-guided algorithm discovery with theoretical analysis. Starting from an LLM-discovered reshuffling rule, we isolate two core components, block reshuffling and paired reversal, and show that they address complementary sources of inefficiency in epoch-based SGD. Our results establish that block reshuffling reduces prefix-gradient variance constants within the unified shuffling framework, while paired reversal cancels the leading order-dependent second-order term and reduces order sensitivity. Empirical results on convex and nonconvex benchmarks validate these insights and show consistent improvements over standard shuffling schemes, particularly outside the small learning rate regime.

Overall, this work takes a first step toward understanding how structured data ordering can improve optimization beyond random reshuffling. Our findings indicate that block reshuffling and paired reversal reflect underlying variance and symmetry principles in without-replacement SGD. These principles suggest a broader design space of structured permutations that remain unexplored, and further investigation may lead to new algorithms and sharper convergence theory.



\appendix


\appendix
\section{Proofs for Section~\ref{sec:advantages-main}}
\label{sec:appendix-proofs}

\subsection{Proof of Lemma~\ref{lem:exact-decomp-main}}
\label{app:proof-exact-decomp}

\begin{proof}
Fix $w\in\mathbb{R}^d$ and a block $B_r$. For any $i\in B_r$,
\begin{align*}
    g_i(w)-\nabla F(w) = \bigl(g_i(w)-G_r(w)\bigr)+\bigl(G_r(w)-\nabla F(w)\bigr).
\end{align*}
Taking squared norms and summing over $i\in B_r$ yields
\begin{align*}
\sum_{i\in B_r}\|g_i(w)-\nabla F(w)\|^2
&=
\sum_{i\in B_r}\|g_i(w)-G_r(w)\|^2
+2\Big\langle \sum_{i\in B_r}(g_i(w)-G_r(w)),\,G_r(w)-\nabla F(w)\Big\rangle\\
&\qquad\qquad\qquad
+\sum_{i\in B_r}\|G_r(w)-\nabla F(w)\|^2.
\end{align*}
By the definition of $G_r(w)$, we have $\sum_{i\in B_r}(g_i(w)-G_r(w))=0$, so the cross term vanishes. Hence
\[
\sum_{i\in B_r}\|g_i(w)-\nabla F(w)\|^2
=
\sum_{i\in B_r}\|g_i(w)-G_r(w)\|^2
+
b\,\|G_r(w)-\nabla F(w)\|^2.
\]
Divide by $b$, then average over $r=1,\dots,K$ and use $n=Kb$:
\begin{align*}
\frac{1}{n}\sum_{i=1}^n\|g_i(w)-\nabla F(w)\|^2
&=
\frac{1}{K}\sum_{r=1}^K \frac{1}{b}\sum_{i\in B_r}\|g_i(w)-G_r(w)\|^2
+
\frac{1}{K}\sum_{r=1}^K \|G_r(w)-\nabla F(w)\|^2\\
&=\sigma_{\mathrm{within}}^2(w)+\sigma_{\mathrm{blk}}^2(w),
\end{align*}
which is exactly \eqref{eq:exact-decomp-main}. The inequality $\sigma_{\mathrm{blk}}^2(w)\le \sigma_{\mathrm{ind}}^2(w)$ follows immediately since $\sigma_{\mathrm{within}}^2(w)\ge 0$, and it is strict whenever $\sigma_{\mathrm{within}}^2(w)>0$.
\end{proof}

\subsection{Proof of Proposition~\ref{prop:rr-prefix-main}}
\label{app:proof-rr-prefix}

\begin{proof}
Apply Lemma~\ref{lem:wor-main} with $m=n$ and $X_i=g_i(w)$. Then $\bar X=\nabla F(w)$ and
$\sigma^2=\sigma_{\mathrm{ind}}^2(w)$ by \eqref{eq:def-sigma-ind-main}. For $\widehat g_m(w)=\frac{1}{m}\sum_{t=1}^m g_{\pi(t)}(w)$,
Lemma~\ref{lem:wor-main} yields \eqref{eq:rr-prefix-main}.
\end{proof}

\subsection{Proof of Proposition~\ref{prop:block-prefix-main}}
\label{app:proof-block-prefix}

\begin{proof}
Apply Lemma~\ref{lem:wor-main} with $m=K$ and $X_r=G_r(w)$. Then $\bar X=\nabla F(w)$ and
$\sigma^2=\sigma_{\mathrm{blk}}^2(w)$ by \eqref{eq:def-sigma-blk-main}. For $\widehat G_k(w)=\frac{1}{k}\sum_{j=1}^k G_{\sigma(j)}(w)$,
Lemma~\ref{lem:wor-main} yields \eqref{eq:block-prefix-main}.
\end{proof}

\subsection{Proof of Lemma~\ref{lem:epoch-expansion-main}}
\label{app:proof-epoch-expansion}

\begin{proof}
Fix $w\in\mathbb{R}^d$ and $\pi\in\mathcal{S}_n$. Define the within-epoch iterates
\[
w_0:=w,
\qquad
w_t:=w_{t-1}-\gamma \nabla f_{\pi(t)}(w_{t-1}),\quad t=1,\dots,n,
\qquad
T_\pi(w):=w_n.
\]

\paragraph{Step 1: Telescoping representation.}
From the recursion,
\begin{equation}
T_\pi(w)-w
=
w_n-w_0
=
-\gamma\sum_{t=1}^n \nabla f_{\pi(t)}(w_{t-1}).
\label{eq:telescoping-app}
\end{equation}

\paragraph{Step 2: Second-order expansion of each gradient around $w$.}
Fix $t\in\{1,\dots,n\}$. By the fundamental theorem of calculus applied to the gradient map,
\begin{equation}
\nabla f_{\pi(t)}(w_{t-1})
=
\nabla f_{\pi(t)}(w)
+
\int_0^1 \nabla^2 f_{\pi(t)}\!\big(w+\tau(w_{t-1}-w)\big)\,(w_{t-1}-w)\,d\tau.
\label{eq:grad-ftc-app}
\end{equation}
Note that $g_i:=\nabla f_i(w)$ and $H_i:=\nabla^2 f_i(w)$, $i = 1,\dots,n$. Add and subtract $H_{\pi(t)}(w)(w_{t-1}-w)$ inside the integral to obtain
\begin{equation}
\nabla f_{\pi(t)}(w_{t-1})
=
g_{\pi(t)} + H_{\pi(t)}(w)(w_{t-1}-w) + r_t,
\label{eq:update_grad_f}
\end{equation}
where
\begin{equation}
r_t
:=
\int_0^1
\Big(\nabla^2 f_{\pi(t)}\!\big(w+\tau(w_{t-1}-w)\big)-\nabla^2 f_{\pi(t)}(w)\Big)\,(w_{t-1}-w)\,d\tau.
\label{eq:def-rt-app}
\end{equation}
Using the Hessian-Lipschitz condition \eqref{eq:A3-main},
\[
\Big\|\nabla^2 f_{\pi(t)}\!\big(w+\tau(w_{t-1}-w)\big)-\nabla^2 f_{\pi(t)}(w)\Big\|
\le
\rho\,\tau\|w_{t-1}-w\|,
\]
hence
\begin{equation}
\|r_t\|
\le \int_0^1
\Big\| \nabla^2 f_{\pi(t)}\!\big(w+\tau(w_{t-1}-w)\big)-\nabla^2 f_{\pi(t)}(w) \Big\| \, \| w_{t-1}-w \| \,d\tau \le
\int_0^1 \rho\,\tau \|w_{t-1}-w\|^2\,d\tau
=
\frac{\rho}{2}\|w_{t-1}-w\|^2.
\label{eq:rt-bound-app}
\end{equation}

Substituting \eqref{eq:update_grad_f} into \eqref{eq:telescoping-app} gives
\begin{equation}
T_\pi(w)-w
=
-\gamma\sum_{t=1}^n g_{\pi(t)}
-\gamma\sum_{t=1}^n H_{\pi(t)}(w)(w_{t-1}-w)
-\gamma\sum_{t=1}^n r_t.
\label{eq:after-taylor-app}
\end{equation}
Since $\sum_{t=1}^n g_{\pi(t)}=\sum_{i=1}^n g_i$ for any permutation, the first term is the desired first-order term.

\paragraph{Step 3: First-order expansion of $w_{t-1}-w$ and collection of second-order terms.}
From the recursion,
\begin{equation}
w_{t-1}-w
=
-\gamma\sum_{s=1}^{t-1}\nabla f_{\pi(s)}(w_{s-1}).
\label{eq:w-diff-exact-app}
\end{equation}

We have
\[
\nabla f_{\pi(s)}(w_{s-1})
=
g_{\pi(s)} + \big(\nabla f_{\pi(s)}(w_{s-1})-\nabla f_{\pi(s)}(w)\big).
\]
Define the accumulated error
\begin{equation}
\delta_{t-1}
:=
-\gamma\sum_{s=1}^{t-1}\big(\nabla f_{\pi(s)}(w_{s-1})-\nabla f_{\pi(s)}(w)\big),
\label{eq:def-delta-app}
\end{equation}
so that
\begin{equation}
w_{t-1}-w
=
-\gamma\sum_{s=1}^{t-1} g_{\pi(s)} + \delta_{t-1}.
\label{eq:w-diff-first-app}
\end{equation}
Plugging \eqref{eq:w-diff-first-app} into the Hessian term in \eqref{eq:after-taylor-app} yields
\begin{align}
-\gamma\sum_{t=1}^n H_{\pi(t)}(w)(w_{t-1}-w)
&=
-\gamma\sum_{t=1}^n H_{\pi(t)}(w)\Big(-\gamma\sum_{s=1}^{t-1} g_{\pi(s)}+\delta_{t-1}\Big)
\nonumber\\
&=
\gamma^2\sum_{t=1}^n\sum_{s=1}^{t-1} H_{\pi(t)}(w)\,g_{\pi(s)}
-\gamma\sum_{t=1}^n H_{\pi(t)}(w)\,\delta_{t-1}
\nonumber\\
&=
\gamma^2\sum_{1\le s<t\le n} H_{\pi(t)}\,g_{\pi(s)}
-\gamma\sum_{t=1}^n H_{\pi(t)}(w)\,\delta_{t-1}.
\label{eq:second-order-plus-error-app}
\end{align}

\paragraph{Step 4: Define and bound the remainder.}
Combine \eqref{eq:after-taylor-app} and \eqref{eq:second-order-plus-error-app} and define
\begin{equation}
R_\pi(w)
:=
-\gamma\sum_{t=1}^n r_t
-\gamma\sum_{t=1}^n H_{\pi(t)}(w)\,\delta_{t-1}.
\label{eq:def-R-app}
\end{equation}
Then \eqref{eq:epoch-expansion-main} holds.

We now bound $\|R_\pi(w)\|$.
First, using \eqref{eq:rt-bound-app},
\begin{equation}
\gamma\sum_{t=1}^n \|r_t\|
\le
\frac{\rho\gamma}{2}\sum_{t=1}^n \|w_{t-1}-w\|^2.
\label{eq:R1-app}
\end{equation}
By \eqref{eq:A2-main}, $\|w_t-w_{t-1}\|=\gamma\|\nabla f_{\pi(t)}(w_{t-1})\|\le \gamma G$, hence
\begin{equation}
\|w_{t-1}-w\|
\le
\sum_{j=1}^{t-1}\|w_j-w_{j-1}\|
\le
(t-1)\gamma G.
\label{eq:w-dist-bound-app}
\end{equation}
Therefore,
\[
\sum_{t=1}^n \|w_{t-1}-w\|^2
\le
\gamma^2 G^2\sum_{t=1}^n (t-1)^2
\le
\gamma^2 G^2\,n^3,
\]
and \eqref{eq:R1-app} gives
\begin{equation}
\gamma\sum_{t=1}^n \|r_t\|
\le
\frac{\rho}{2}\gamma^3 G^2 n^3.
\label{eq:R1-final-app}
\end{equation}

Second, bound the term involving $\delta_{t-1}$. By \eqref{eq:def-delta-app} and \eqref{eq:A1-main},
\[
\|\delta_{t-1}\|
\le
\gamma\sum_{s=1}^{t-1}\|\nabla f_{\pi(s)}(w_{s-1})-\nabla f_{\pi(s)}(w)\|
\le
\gamma\sum_{s=1}^{t-1} L\|w_{s-1}-w\|.
\]
Using \eqref{eq:w-dist-bound-app},
\[
\|\delta_{t-1}\|
\le
\gamma\sum_{s=1}^{t-1} L (s-1)\gamma G
=
LG\gamma^2\frac{(t-1)(t-2)}{2}
\le
\frac{LG}{2}\gamma^2 (t-1)^2.
\]
Moreover, since $\nabla f_i$ is $L$-Lipschitz and $f_i$ is twice differentiable, we have $\|H_i(x)\|\le L$ for all $x$, hence $\|H_{\pi(t)}(w)\|\le L$.
Thus
\begin{align}
\gamma\sum_{t=1}^n \|H_{\pi(t)}(w)\delta_{t-1}\|
&\le
\gamma\sum_{t=1}^n L\|\delta_{t-1}\|
\le
\gamma\sum_{t=1}^n L\cdot \frac{LG}{2}\gamma^2 (t-1)^2
=
\frac{L^2G}{2}\gamma^3\sum_{t=1}^n (t-1)^2
\le
\frac{L^2G}{2}\gamma^3 n^3.
\label{eq:R2-final-app}
\end{align}

Finally, combining \eqref{eq:def-R-app}, \eqref{eq:R1-final-app}, and \eqref{eq:R2-final-app},
\[
\|R_\pi(w)\|
\le
\left(\frac{\rho G^2}{2}+\frac{L^2G}{2}\right)\gamma^3 n^3
=
C_{\mathrm{rem}}\,\gamma^3 n^3,
\]
which matches \eqref{eq:rem-bound-main}.
\end{proof}

\subsection{Proof of Theorem~\ref{thm:paired-rev-main}}
\label{app:proof-paired-rev}

\begin{proof}
Start from the expansions in Lemma~\ref{lem:epoch-expansion-main} for $\pi$ and for $\mathrm{Rev}(\pi)$:
\[
T_\pi(w)
=
w-\gamma\sum_{i=1}^n g_i
+\gamma^2 B_\pi
+R_\pi(w),
\qquad
T_{\mathrm{Rev}(\pi)}(w)
=
w-\gamma\sum_{i=1}^n g_i
+\gamma^2 B_{\mathrm{Rev}(\pi)}
+R_{\mathrm{Rev}(\pi)}(w),
\]
where
\[
B_\pi:=\sum_{1\le s<t\le n} H_{\pi(t)}g_{\pi(s)}.
\]
Averaging gives
\[
\bar T_\pi(w)=\frac12\big(T_\pi(w)+T_{\mathrm{Rev}(\pi)}(w)\big)
=
w-\gamma\sum_{i=1}^n g_i
+\frac{\gamma^2}{2}\big(B_\pi+B_{\mathrm{Rev}(\pi)}\big)
+\bar R_\pi(w),
\]
where $\bar R_\pi(w):=\frac12(R_\pi(w)+R_{\mathrm{Rev}(\pi)}(w))$ so
$\|\bar R_\pi(w)\|\le C_{\mathrm{rem}}\gamma^3 n^3$.

It remains to show $B_\pi+B_{\mathrm{Rev}(\pi)}=\sum_{i\neq j}H_i g_j$, which is independent of $\pi$.
Using the index substitution $t'=n+1-t$, $s'=n+1-s$ (so that $s<t \Leftrightarrow t'<s'$),
\[
B_{\mathrm{Rev}(\pi)}
=
\sum_{1\le s<t\le n} H_{\mathrm{Rev}(\pi)(t)}\,g_{\mathrm{Rev}(\pi)(s)}
=
\sum_{1\le s<t\le n} H_{\pi(n+1-t)}\,g_{\pi(n+1-s)}
=
\sum_{1\le t'<s'\le n} H_{\pi(t')}\,g_{\pi(s')}.
\]
Therefore,
\[
B_\pi+B_{\mathrm{Rev}(\pi)}
=
\sum_{1\le s<t\le n} H_{\pi(t)}g_{\pi(s)}
+
\sum_{1\le t<s\le n} H_{\pi(t)}g_{\pi(s)}
=
\sum_{\substack{1\le s,t\le n\\ s\neq t}} H_{\pi(t)}g_{\pi(s)}
=
\sum_{i\neq j} H_i g_j,
\]
independent of $\pi$. This proves the expansion \eqref{eq:paired-rev-expansion-main}.

Finally, for any $\pi,\pi'$,
\[
\bar T_\pi(w)-\bar T_{\pi'}(w)=\bar R_\pi(w)-\bar R_{\pi'}(w),
\]
so by the triangle inequality and $\|\bar R_\pi(w)\|\le C_{\mathrm{rem}}\gamma^3 n^3$,
\[
\|\bar T_\pi(w)-\bar T_{\pi'}(w)\|
\le
\|\bar R_\pi(w)\|+\|\bar R_{\pi'}(w)\|
\le
2C_{\mathrm{rem}}\gamma^3 n^3,
\]
which is \eqref{eq:order-sensitivity-main}.
\end{proof}

\subsection{Proof of Theorem~\ref{thm:order-upper-main}}
\label{app:proof-order-upper}

\begin{proof}
Fix $w\in\mathbb{R}^d$ and two permutations $\pi,\pi'\in\mathcal{S}_n$.
By Lemma~\ref{lem:epoch-expansion-main}, we have the second-order expansions
\begin{align*}
T_\pi(w)
&=
w-\gamma\sum_{i=1}^n g_i
+\gamma^2 B_\pi(w)
+R_\pi(w),
\\
T_{\pi'}(w)
&=
w-\gamma\sum_{i=1}^n g_i
+\gamma^2 B_{\pi'}(w)
+R_{\pi'}(w),
\end{align*}
where
\[
B_\pi(w):=\sum_{1\le s<t\le n} H_{\pi(t)}\,g_{\pi(s)},
\qquad
\|R_\pi(w)\|\le C_{\mathrm{rem}}\gamma^3 n^3,
\qquad
\|R_{\pi'}(w)\|\le C_{\mathrm{rem}}\gamma^3 n^3.
\]
Subtracting yields
\[
T_\pi(w)-T_{\pi'}(w)
=
\gamma^2\big(B_\pi(w)-B_{\pi'}(w)\big)
+
\big(R_\pi(w)-R_{\pi'}(w)\big).
\]
Taking norms and using the triangle inequality gives
\begin{equation}
\|T_\pi(w)-T_{\pi'}(w)\|
\le
\gamma^2\|B_\pi(w)-B_{\pi'}(w)\|
+
\|R_\pi(w)\|+\|R_{\pi'}(w)\|.
\label{eq:U1-app}
\end{equation}

It remains to bound $\|B_\pi(w)\|$ uniformly over $\pi$.
Since each $f_i$ is twice differentiable and $\nabla f_i$ is $L$-Lipschitz by \eqref{eq:A1-main},
we have $\|H_i(x)\|\le L$ for all $x$ (in particular at $x=w$). Moreover, \eqref{eq:A2-main} implies $\|g_i\|=\|\nabla f_i(w)\|\le G$.
Therefore,
\begin{align*}
\|B_\pi(w)\|
&=
\left\|\sum_{1\le s<t\le n} H_{\pi(t)}\,g_{\pi(s)}\right\|
\le
\sum_{1\le s<t\le n} \|H_{\pi(t)}\|\,\|g_{\pi(s)}\|
\le
\sum_{1\le s<t\le n} L\,G
=
LG\cdot \frac{n(n-1)}{2}.
\end{align*}
By the same bound for $\pi'$, we get
\begin{equation}
\|B_\pi(w)-B_{\pi'}(w)\|
\le
\|B_\pi(w)\|+\|B_{\pi'}(w)\|
\le
LG\,n(n-1).
\label{eq:U2-app}
\end{equation}

Finally, combining \eqref{eq:U1-app}--\eqref{eq:U2-app} with $\|R_\pi(w)\|,\|R_{\pi'}(w)\|\le C_{\mathrm{rem}}\gamma^3 n^3$,
\[
\|T_\pi(w)-T_{\pi'}(w)\|
\le
LG\,\gamma^2\,n(n-1)
+
2C_{\mathrm{rem}}\gamma^3 n^3,
\]
which is exactly \eqref{eq:order-upper-main}.
\end{proof}

\subsection{Proof of Proposition~\ref{prop:order-sens-lb-main}}
\label{app:proof-order-lb}

\begin{proof}
Let $n=2$ and consider the one-dimensional functions
\[
f_1(x)=\frac{a}{2}x^2,
\qquad
f_2(x)=bx,
\]
with $a\neq 0$ and $b\neq 0$. Fix any $w\neq 0$ and stepsize $\gamma>0$.

For $\pi=(1,2)$:
\[
x_1=w-\gamma\nabla f_1(w)=w-\gamma a w=(1-\gamma a)w,
\qquad
x_2=x_1-\gamma\nabla f_2(x_1)=(1-\gamma a)w-\gamma b,
\]
so $T_{(1,2)}(w)=(1-\gamma a)w-\gamma b$.

For $\pi'=(2,1)$:
\[
x_1'=w-\gamma\nabla f_2(w)=w-\gamma b,
\qquad
x_2'=x_1'-\gamma\nabla f_1(x_1')
=(w-\gamma b)-\gamma a(w-\gamma b)
=(1-\gamma a)w-\gamma b+\gamma^2 ab,
\]
so $T_{(2,1)}(w)=(1-\gamma a)w-\gamma b+\gamma^2 ab$.

Thus,
\[
T_{(2,1)}(w)-T_{(1,2)}(w)=\gamma^2 ab,
\]
and therefore
\[
\big|T_{(2,1)}(w)-T_{(1,2)}(w)\big|=|ab|\,\gamma^2.
\]
This establishes the claim with $c=|ab|$.
\end{proof}

\subsection{Proof of Theorem~\ref{thm:perm-var-upper-main}}
\label{app:proof-permvar-upper}

\begin{proof}
From Lemma~\ref{lem:epoch-expansion-main}, write
\[
T_\pi(w)=A(w)+\gamma^2 B_\pi(w)+R_\pi(w),
\qquad
A(w):=w-\gamma\sum_{i=1}^n g_i,
\qquad
B_\pi(w):=\sum_{1\le s<t\le n} H_{\pi(t)}g_{\pi(s)}.
\]
Since $A(w)$ is deterministic, subtracting expectations yields
\[
T_\pi(w)-\mathbb{E}_\pi[T_\pi(w)]
=
\gamma^2\big(B_\pi(w)-\mathbb{E}B_\pi(w)\big)
+
\big(R_\pi(w)-\mathbb{E}R_\pi(w)\big).
\]
Using $\|x+y\|^2\le 2\|x\|^2+2\|y\|^2$ and $\mathbb{E}\|R_\pi-\mathbb{E}R_\pi\|^2\le \mathbb{E}\|R_\pi\|^2$,
\begin{align*}
\mathrm{Var}_\pi\!\big(T_\pi(w)\big)
&=
\mathbb{E}_\pi\big\|T_\pi(w)-\mathbb{E}_\pi[T_\pi(w)]\big\|^2\\
&\le
2\gamma^4\,\mathbb{E}_\pi\big\|B_\pi(w)-\mathbb{E}_\pi[B_\pi(w)]\big\|^2
+
2\,\mathbb{E}_\pi\|R_\pi(w)\|^2\\
&=
2\gamma^4\,\mathrm{Var}_\pi\!\big(B_\pi(w)\big)
+
2\,\mathbb{E}_\pi\|R_\pi(w)\|^2
\;\le\;
2\gamma^4\,\mathrm{Var}_\pi\!\big(B_\pi(w)\big)
+
2C_{\mathrm{rem}}^2\gamma^6 n^6,
\end{align*}
which gives \eqref{eq:var-T-upper-main}.

For the paired-reversal map, from Theorem~\ref{thm:paired-rev-main} we have
\[
\bar T_\pi(w)=\bar A(w)+\bar R_\pi(w),
\]
where $\bar A(w)$ is deterministic (independent of $\pi$). Hence
\[
\mathrm{Var}_\pi\!\big(\bar T_\pi(w)\big)
=
\mathbb{E}_\pi\big\|\bar R_\pi(w)-\mathbb{E}_\pi[\bar R_\pi(w)]\big\|^2
\le
\mathbb{E}_\pi\|\bar R_\pi(w)\|^2
\le
C_{\mathrm{rem}}^2\gamma^6 n^6,
\]
which is \eqref{eq:var-Tbar-upper-main}.
\end{proof}

\subsection{Proof of Proposition~\ref{prop:permvar-lb-main}}
\label{app:proof-permvar-lb}

\begin{proof}
We use a fully explicit $n=2$, one-dimensional instance.
Let
\[
f_1(x)=\frac{a}{2}x^2,
\qquad
f_2(x)=bx,
\]
with $a\neq 0$ and $b\neq 0$. Fix $w\neq 0$ and stepsize $\gamma>0$. Let $\pi$ be uniform over $\mathcal{S}_2$, i.e.,
$\mathbb{P}(\pi=(1,2))=\mathbb{P}(\pi=(2,1))=1/2$.

\paragraph{Bernoulli second-order term.}
For $n=2$, the second-order term is
\[
B_\pi(w)=H_{\pi(2)}(w)\,g_{\pi(1)}(w),
\]
where $g_1(w)=aw$, $H_1(w)=a$, $g_2(w)=b$, $H_2(w)=0$.
Thus,
\[
B_{(1,2)}(w)=H_2(w)g_1(w)=0,
\qquad
B_{(2,1)}(w)=H_1(w)g_2(w)=ab,
\]
so $B_\pi(w)$ is Bernoulli:
\[
B_\pi(w)=
\begin{cases}
0, & \text{w.p. }1/2,\\
ab, & \text{w.p. }1/2.
\end{cases}
\]
Hence
\[
\mathbb{E}B_\pi(w)=\frac{ab}{2},
\qquad
\mathrm{Var}_\pi(B_\pi(w))=\frac{(ab)^2}{4}.
\]

\paragraph{A variance perturbation inequality.}
We use the following inequality: for any square-integrable random vectors $X,Y$,
\begin{equation}
\mathrm{Var}(X+Y)\ge \frac12\,\mathrm{Var}(X)-\mathbb{E}\|Y\|^2.
\label{eq:var-perturb-app}
\end{equation}
To see this, let $\mu_X=\mathbb{E}X$, $\mu_Y=\mathbb{E}Y$ and write
$\mathrm{Var}(X+Y)=\mathbb{E}\|(X-\mu_X)+(Y-\mu_Y)\|^2$.
Using $\|u+v\|^2\ge \frac12\|u\|^2-\|v\|^2$ yields \eqref{eq:var-perturb-app}.

\paragraph{Lower bound for $\mathrm{Var}_\pi(T_\pi(w))$.}
From Lemma~\ref{lem:epoch-expansion-main}, for $n=2$ we have the expansion
\[
T_\pi(w)=A(w)+\gamma^2 B_\pi(w)+R_\pi(w),
\]
where $A(w)$ is deterministic and $\|R_\pi(w)\|\le C_{\mathrm{rem}}\gamma^3 n^3 = 8C_{\mathrm{rem}}\gamma^3$.
Subtracting expectations removes $A(w)$, so
\[
T_\pi(w)-\mathbb{E}T_\pi(w)
=
\gamma^2(B_\pi(w)-\mathbb{E}B_\pi(w))
+
(R_\pi(w)-\mathbb{E}R_\pi(w)).
\]
Apply \eqref{eq:var-perturb-app} with $X=\gamma^2(B_\pi-\mathbb{E}B_\pi)$ and
$Y=R_\pi-\mathbb{E}R_\pi$. Then
\begin{align*}
\mathrm{Var}_\pi(T_\pi(w))
&=\mathrm{Var}_\pi\big(\gamma^2 B_\pi(w)+R_\pi(w)\big)\\
&\ge \frac12\,\mathrm{Var}_\pi\big(\gamma^2 B_\pi(w)\big) - \mathbb{E}_\pi\|R_\pi(w)-\mathbb{E}R_\pi(w)\|^2\\
&\ge \frac12\,\gamma^4\,\mathrm{Var}_\pi\big(B_\pi(w)\big) - \mathbb{E}_\pi\|R_\pi(w)\|^2\\
&\ge \frac12\,\gamma^4\cdot \frac{(ab)^2}{4} - (8C_{\mathrm{rem}}\gamma^3)^2\\
&= \frac{(ab)^2}{8}\gamma^4 - 64C_{\mathrm{rem}}^2\gamma^6.
\end{align*}
In particular, for sufficiently small $\gamma$ (e.g., $64C_{\mathrm{rem}}^2\gamma^2\le (ab)^2/16$),
\[
\mathrm{Var}_\pi(T_\pi(w))\ge \frac{(ab)^2}{16}\gamma^4,
\]
which proves $\mathrm{Var}_\pi(T_\pi(w))=\Omega(\gamma^4)$ for this instance. Together with the general upper bound
$\mathrm{Var}_\pi(T_\pi(w))=O(\gamma^4)+O(\gamma^6)$ from the second-order representation, this establishes
$\mathrm{Var}_\pi(T_\pi(w))=\Theta(\gamma^4)$ for fixed $n=2$.
\end{proof}

\section{More Experimental Results}\label{sec_app_more_experiments}

\subsection{Datasets and Models}
\label{sec:datasets-models}

We evaluate shuffling strategies across a diverse collection of benchmark datasets,
covering both classification and regression tasks for convex and non-convex objectives.
All models are intentionally lightweight, so that performance differences primarily reflect
the effect of data ordering rather than architectural complexity.

\paragraph{Classification datasets.}
We consider the following classification benchmarks:

\begin{itemize}
    \item \textbf{a9a} and \textbf{breast\_cancer} from \texttt{LIBSVM} \citep{LIBSVM}.  
    These datasets are treated as binary classification problems and are trained using
    linear logistic regression models.
    Features are standardized, and the objective is the regularized log-loss.
    Optimization is performed using stochastic gradient descent (SGD) with constant and diminishing learning rates.
    These tasks correspond to convex finite-sum optimization problems and are well suited
    for validating theoretical predictions on variance and order sensitivity.

    \item \textbf{digits} \citep{kaynak1995methods}. 
    The Digits dataset consists of $8\times 8$ grayscale images.
    We binarize the labels (digits $>5$ vs.\ $\le 5$) and use a logistic regression model,
    yielding a convex objective.
    This dataset serves as a low-dimensional image-based benchmark that remains analytically tractable.

    \item \textbf{MNIST} \citep{MNIST} and \textbf{FashionMNIST} \citep{xiao2017fashion}.
    These datasets contain $28\times 28$ grayscale images with 10 classes both with $60,000$ samples.
    We use a simple multilayer perceptron (MLP) with two hidden layers: $784 \;\rightarrow\; 256 \;\rightarrow\; 128 \;\rightarrow\; 10$, 
    with ReLU activations.
    The networks are trained using cross-entropy loss and Adam.
    No convolutional layers, batch normalization, dropout, or data augmentation are used.
    This results in non-convex objectives while keeping the architecture minimal, allowing us
    to isolate the impact of shuffling on optimization dynamics.
\end{itemize}

\paragraph{Regression datasets.}
We also evaluate on standard regression benchmarks:

\begin{itemize}
    \item \textbf{California Housing}, \textbf{Boston Housing}, and \textbf{Diabetes} from \texttt{sklearn.datasets} \citep{pace1997sparse,buitinck2013api}.
    These datasets involve predicting continuous targets from tabular features.
    We use linear regression models trained with mean squared error (MSE) loss and SGD.
    All features are standardized.
    These problems are convex and allow us to test shuffling strategies in regression settings
    complementary to classification.
\end{itemize}

\paragraph{Training process.}
For all experiments, models are trained for $100$ epochs. To account for randomness arising from both initialization and data ordering, each configuration is evaluated using $5$ different random initializations and $5$ independent runs, resulting in $25$ trials in total. We report the mean and standard deviation of the best-so-far training loss across these trials.

For SGD-based experiments, we evaluate a range of constant learning rates 
\[
\{0.5,\, 0.1,\, 0.05,\, 0.01,\, 0.005,\, 0.001,\, 0.0005,\, 0.0001\},
\]
as well as a diminishing learning-rate schedule of the form
$\gamma_t = \gamma_0 / t^{\alpha}$ with $\alpha = 0.5$ with the mini-batch sizes of $\{64, 128, 256\}$.
This range allows us to study shuffling behavior across both moderate and small learning rate regimes,
including the small variance regime predicted by theory. For each dataset and shuffling scheme, we report results corresponding to the best performing constant learning rate selected from this grid, as well as the diminishing schedule.
This process reflects standard practice and ensures that comparisons are not biased by suboptimal step-size choices.

For MNIST and FashionMNIST, we use the Adam optimizer with learning rate $0.001$,
following standard practice.
Unless otherwise specified, all other hyperparameters are kept fixed across shuffling schemes,
ensuring that differences in performance are attributable solely to the data ordering strategy.

\FloatBarrier
\subsection{Large Learning Rates}

This appendix provides additional experimental results that complement and support the findings reported in the main text.

Across these additional experiments, we observe trends consistent with the conclusions drawn in the main text. In particular, the LLM-discovered shuffling strategy continues to outperform or match standard shuffling schemes in terms of best-so-far training loss, while often exhibiting reduced variability across runs. These results further support our claims that structured data ordering can improve optimization performance beyond random reshuffling, especially outside the small learning rate regime.

We emphasize that no new algorithmic components are introduced in this appendix.
Rather, these results serve to demonstrate the robustness and generality of the empirical behavior discussed in the main paper. Taken together with the main results, the experiments in this appendix strengthen the evidence that block reshuffling and reversal-based structure are key contributors to the observed performance gains.

\begin{figure}[!htbp]
    \centering

    \begin{subfigure}[h]{\textwidth}
        \centering
        \includegraphics[width= 0.32\textwidth]{Figs/a9a_constant_large.png}
        \includegraphics[width= 0.32\textwidth]{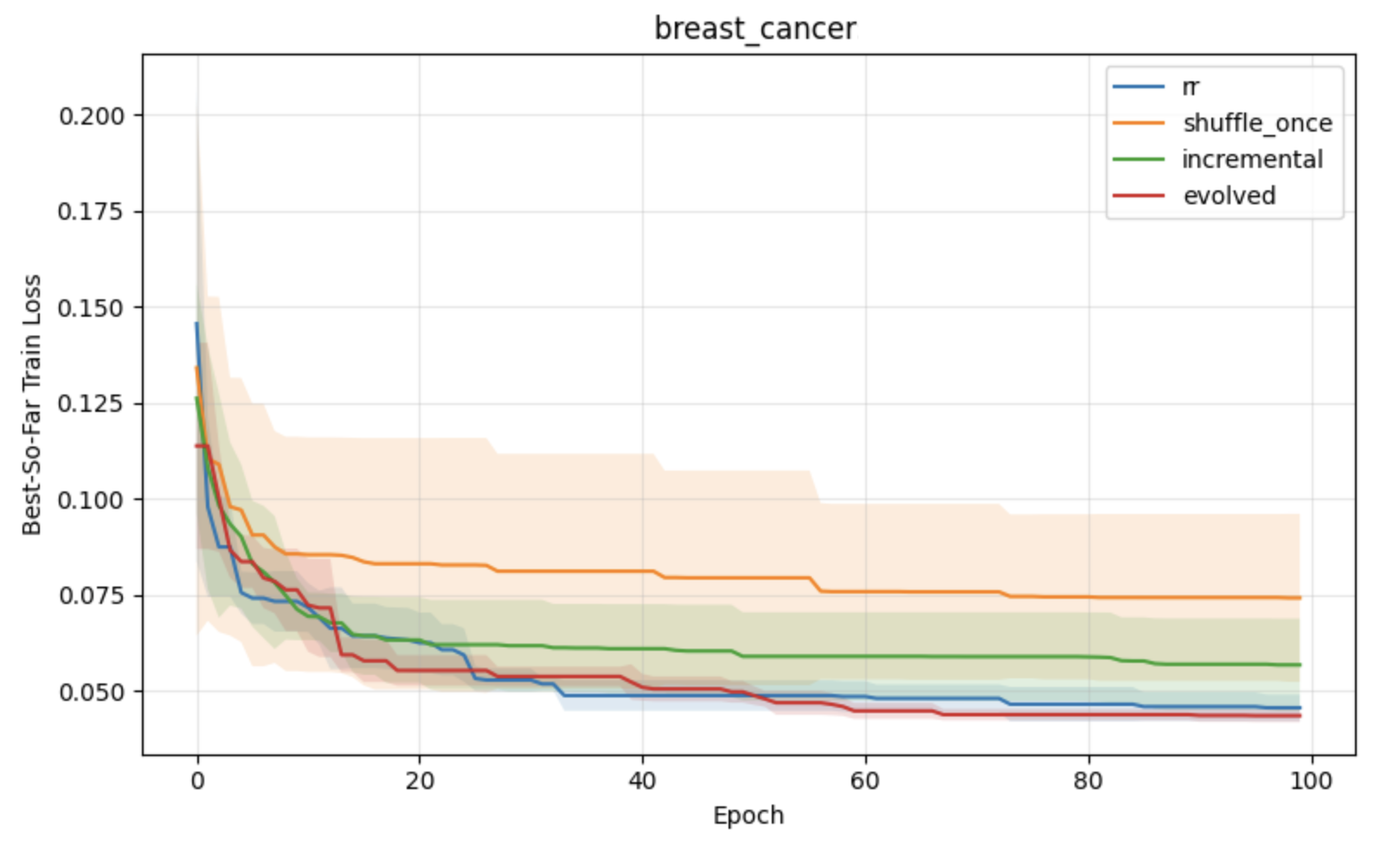}
        \includegraphics[width= 0.32\textwidth]{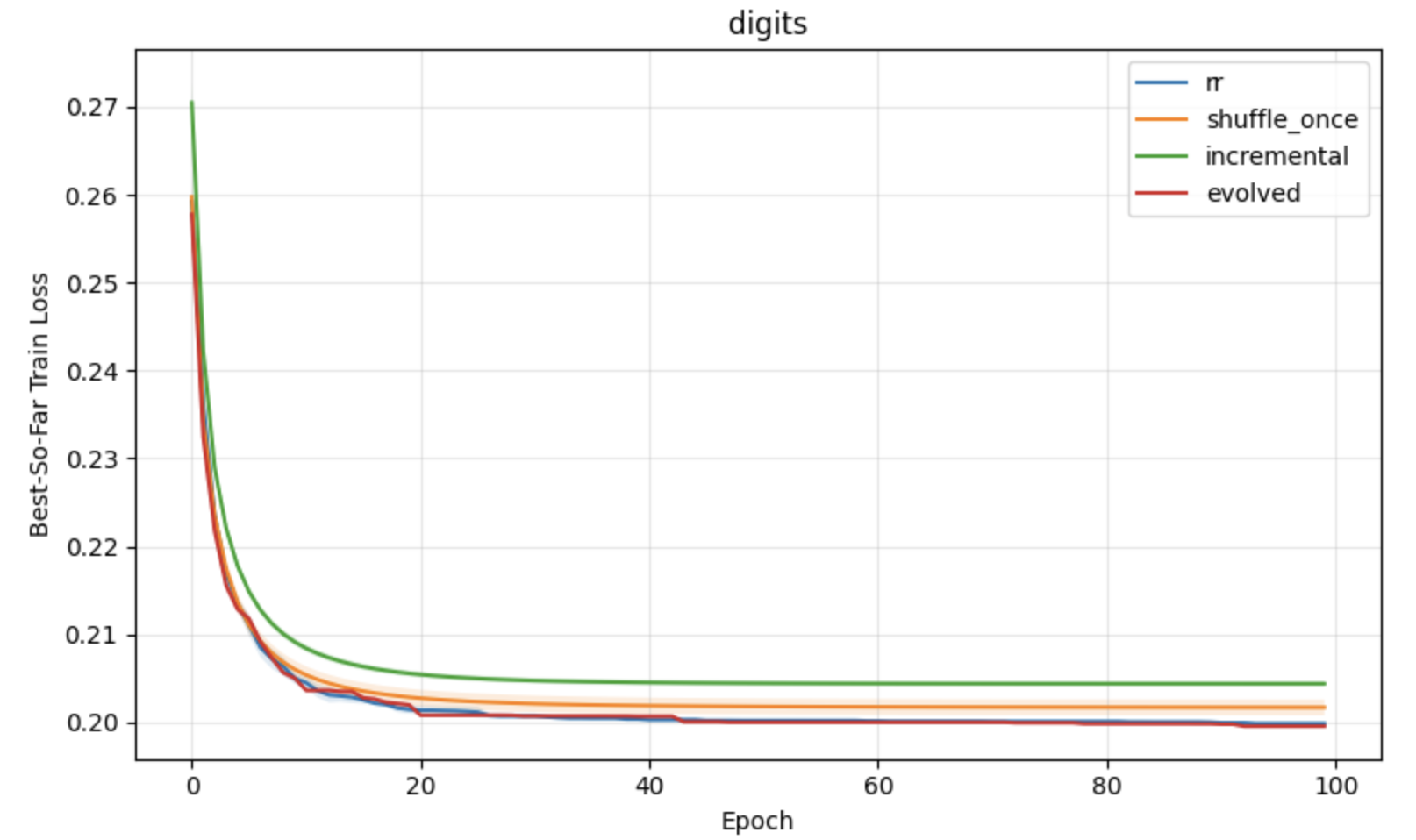}
        \caption{Classification Datasets - Constant Learning Rates}
        \label{fig_classification_constant_lr}
    \end{subfigure}

    \vspace{0.5em}

    \begin{subfigure}[h]{\textwidth}
        \centering
        \includegraphics[width= 0.32\textwidth]{Figs/a9a_poly_large.png}
        \includegraphics[width= 0.32\textwidth]{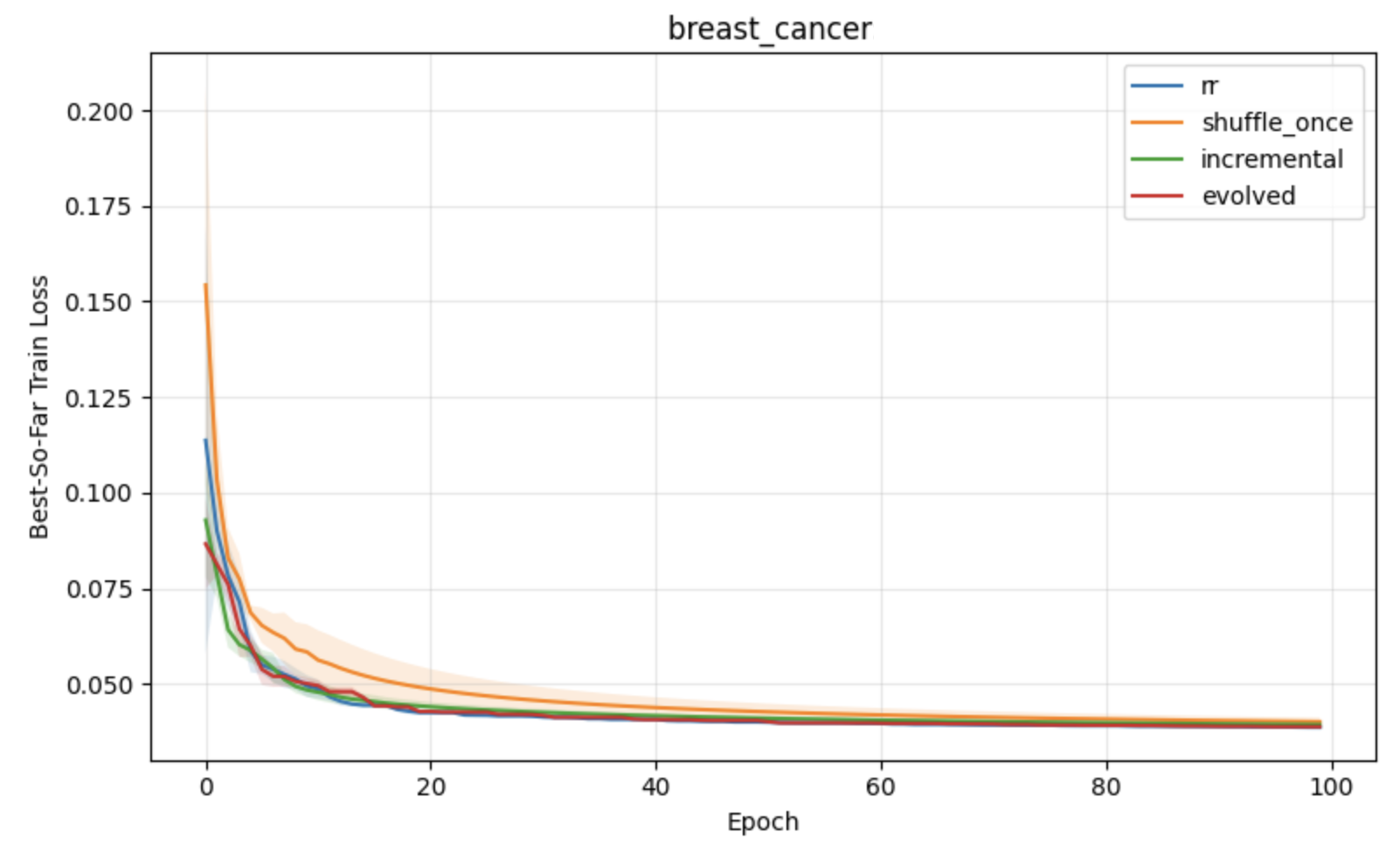}
        \includegraphics[width= 0.32\textwidth]{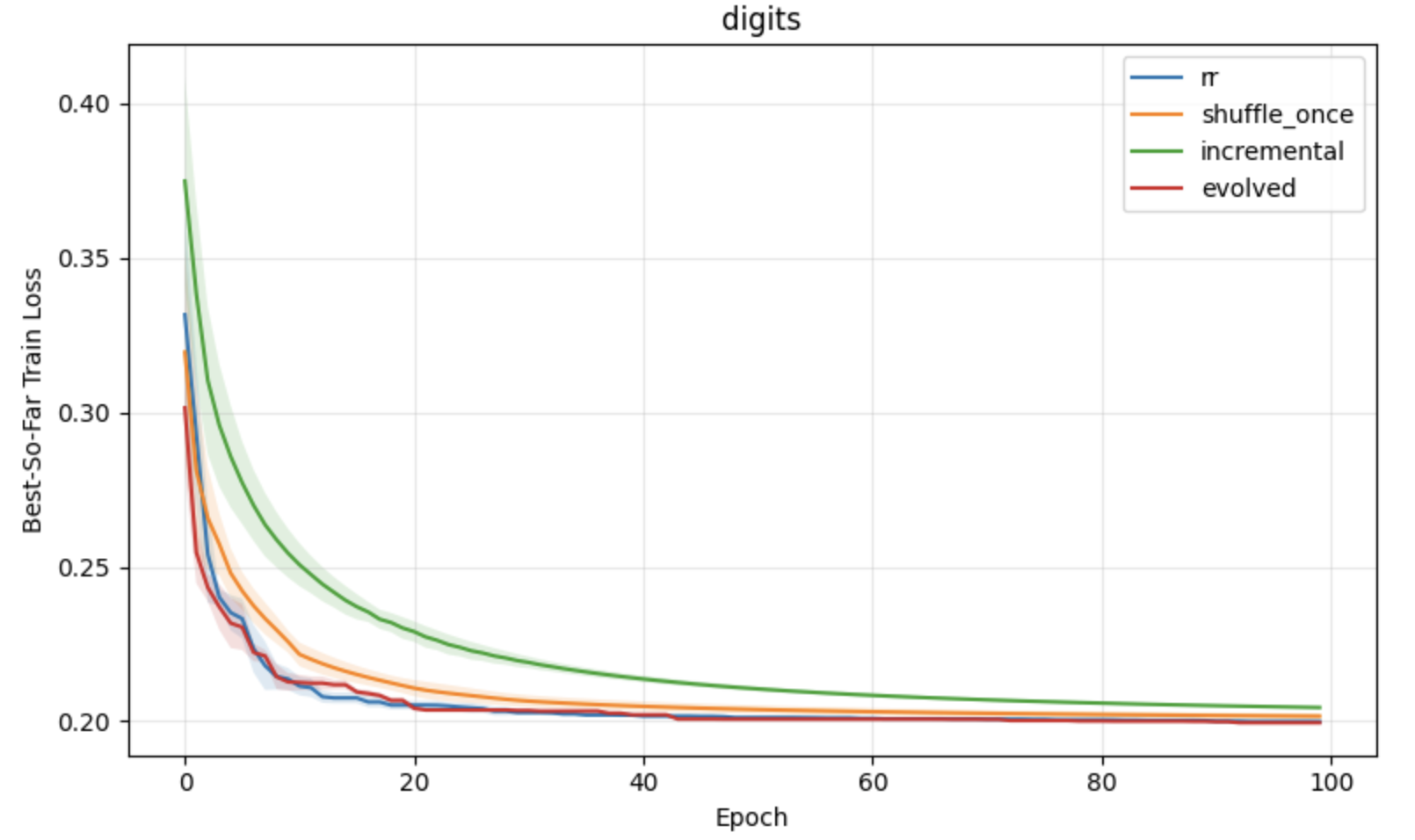}
        \caption{Classification Datasets - Diminishing Learning Rates}
        \label{fig_classification_diminishing_lr}
    \end{subfigure}

    \caption{Classification Datasets}
    \label{fig:experiments_classification_lr_large}
\end{figure}

\begin{figure}[!htbp]
    \centering

    \begin{subfigure}[h]{\textwidth}
        \centering
        \includegraphics[width= 0.32\textwidth]{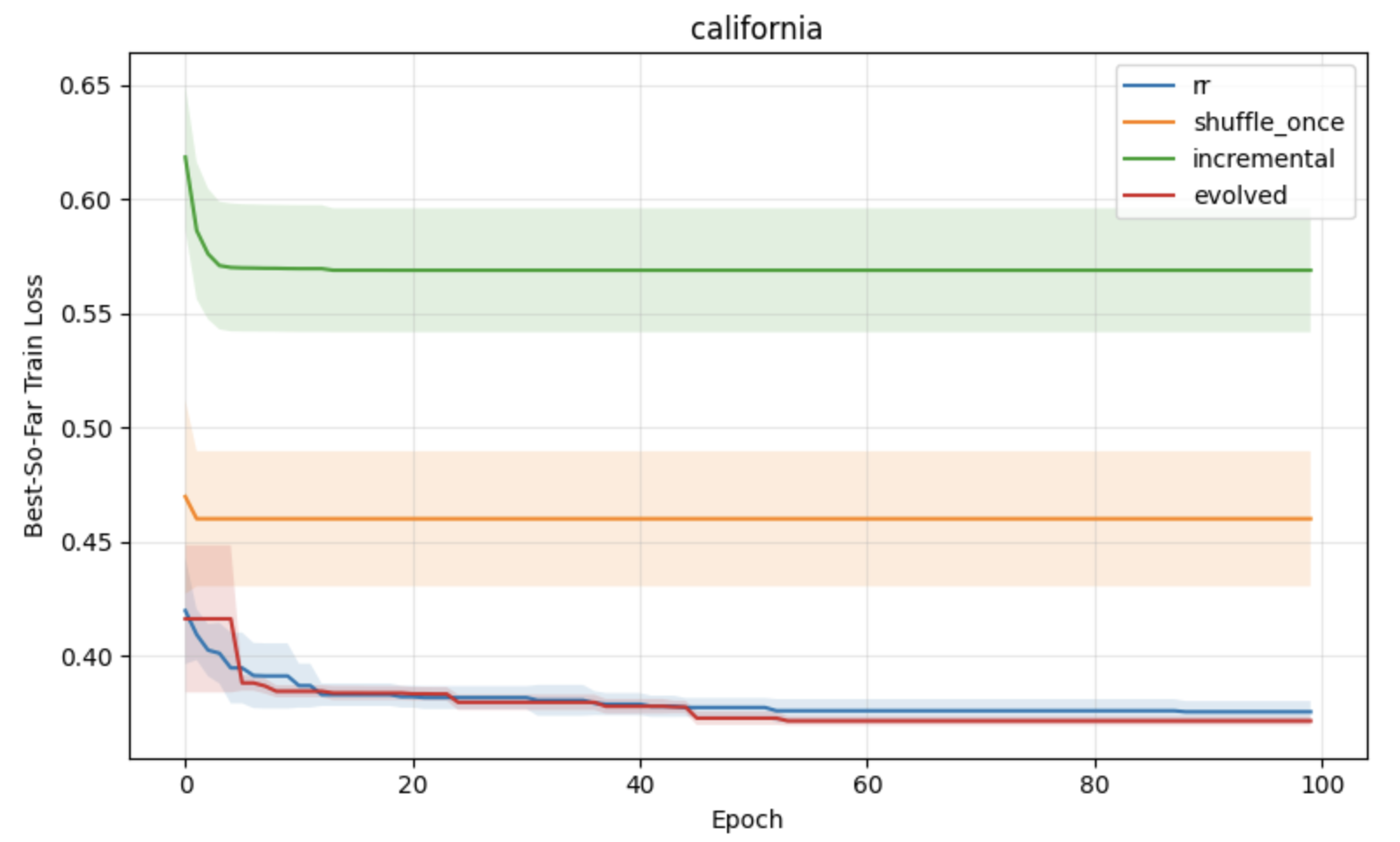}
        \includegraphics[width= 0.32\textwidth]{Figs/boston_constant_large.png}
        \includegraphics[width= 0.32\textwidth]{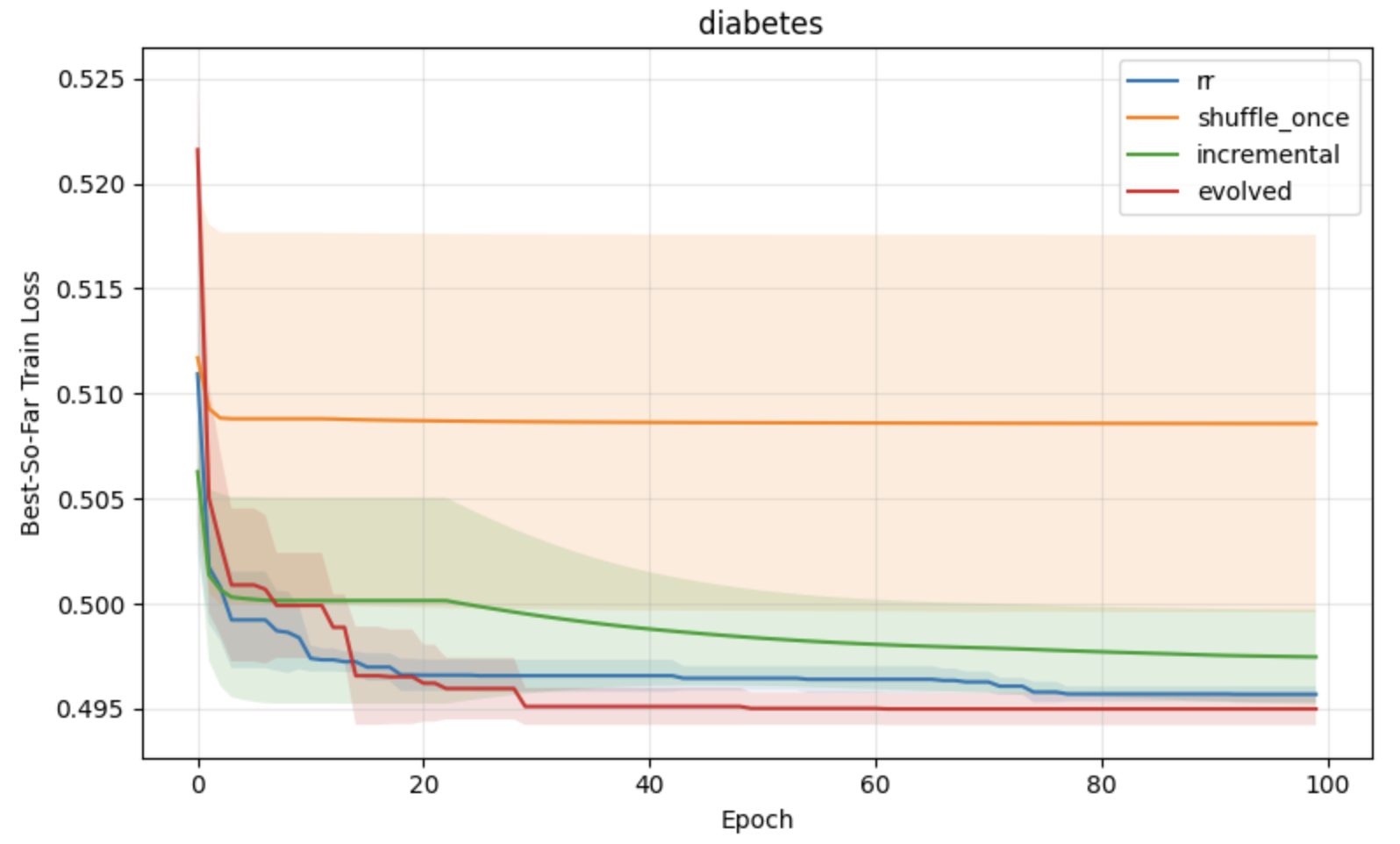}
        \caption{Regression Datasets - Constant Learning Rates}
        \label{fig_regression_constant_lr}
    \end{subfigure}

    \vspace{0.5em}

    \begin{subfigure}[h]{\textwidth}
        \centering
        \includegraphics[width= 0.32\textwidth]{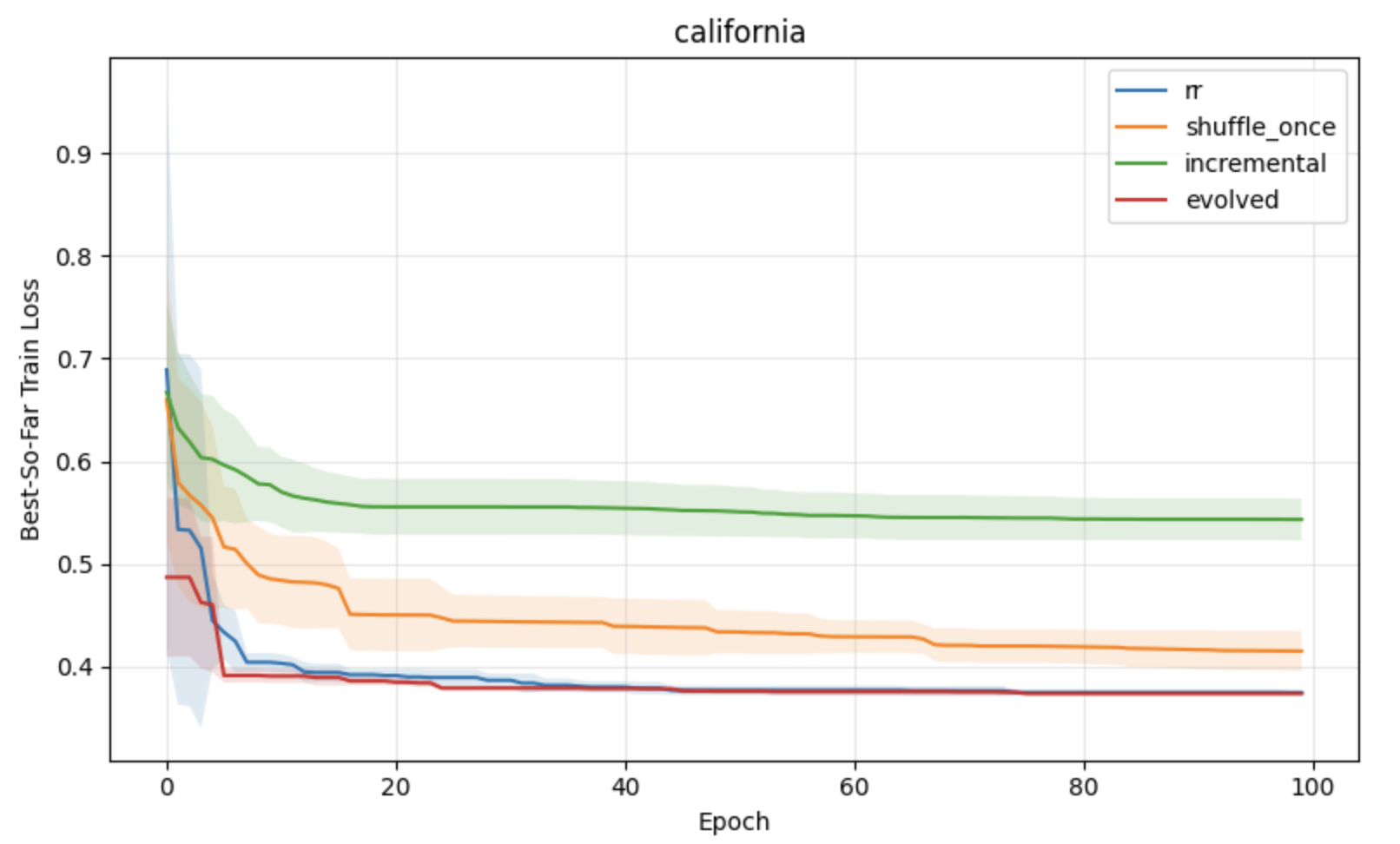}
        \includegraphics[width= 0.32\textwidth]{Figs/boston_poly_large.png}
        \includegraphics[width= 0.32\textwidth]{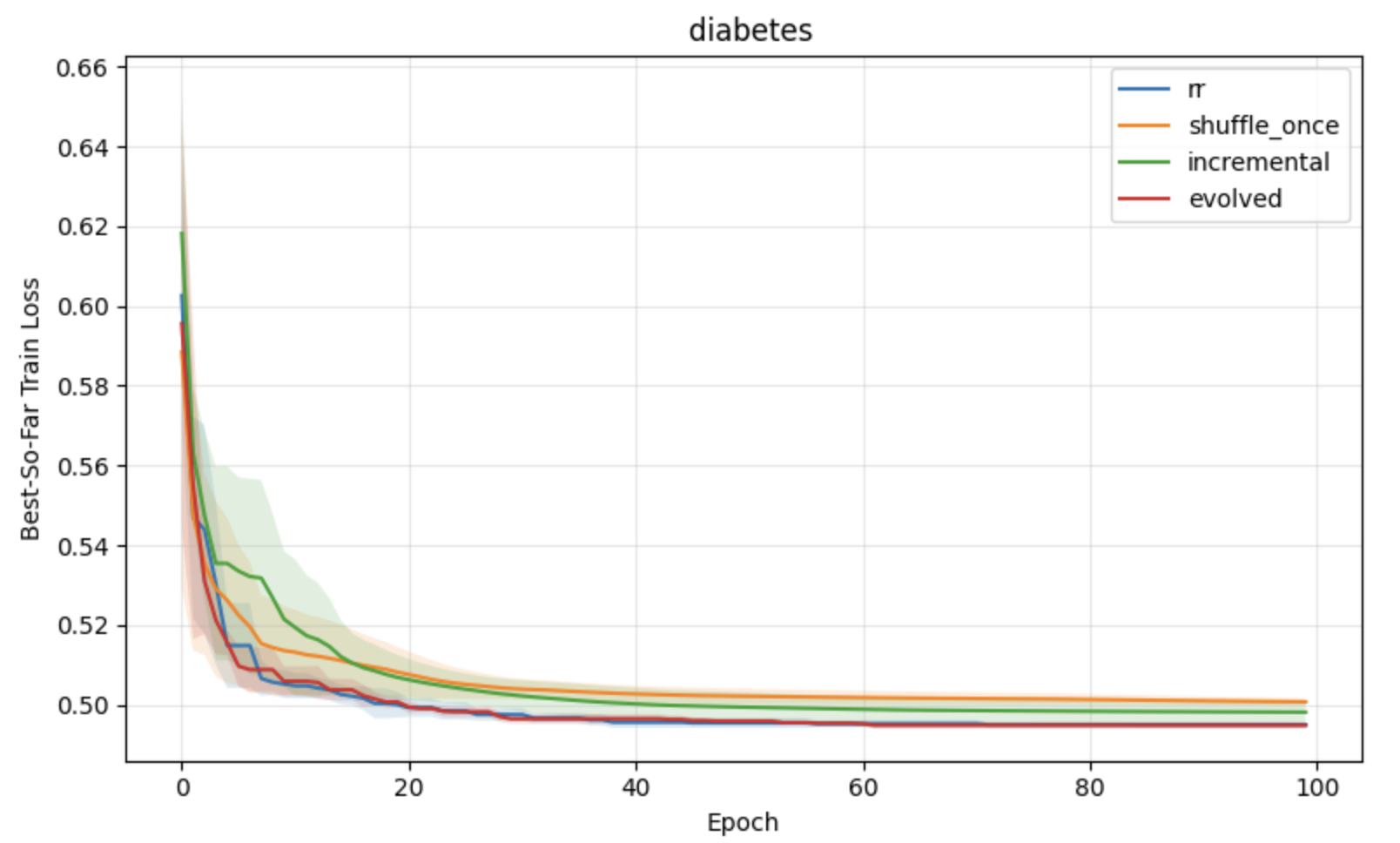}
        \caption{Regression Datasets - Diminishing Learning Rates}
        \label{fig_regression_diminishing_lr}
    \end{subfigure}

    \caption{Regression Datasets}
    \label{fig:experiments_regression_lr_large}
\end{figure}

\begin{figure}[!htbp]
    \centering

    \begin{subfigure}[h]{\textwidth}
        \centering
        \includegraphics[width= 0.45\textwidth]{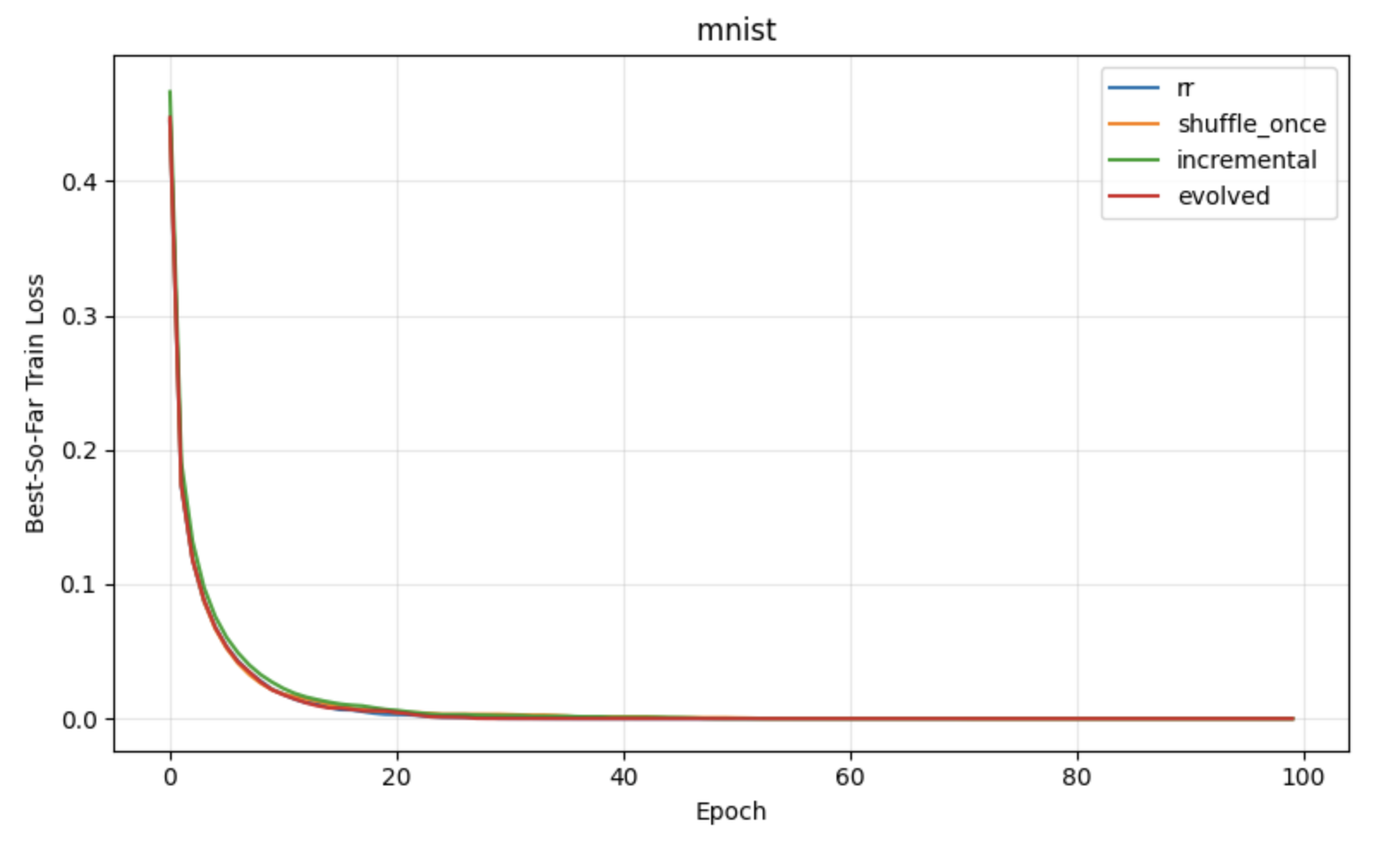}
        \includegraphics[width= 0.45\textwidth]{Figs/fashionmnist_constant_large.png}
        \caption{Classification Datasets (NN) - Constant Learning Rates}
        \label{fig_classification_nn_constant_lr}
    \end{subfigure}

    \vspace{0.5em}

    \begin{subfigure}[h]{\textwidth}
        \centering
        \includegraphics[width= 0.45\textwidth]{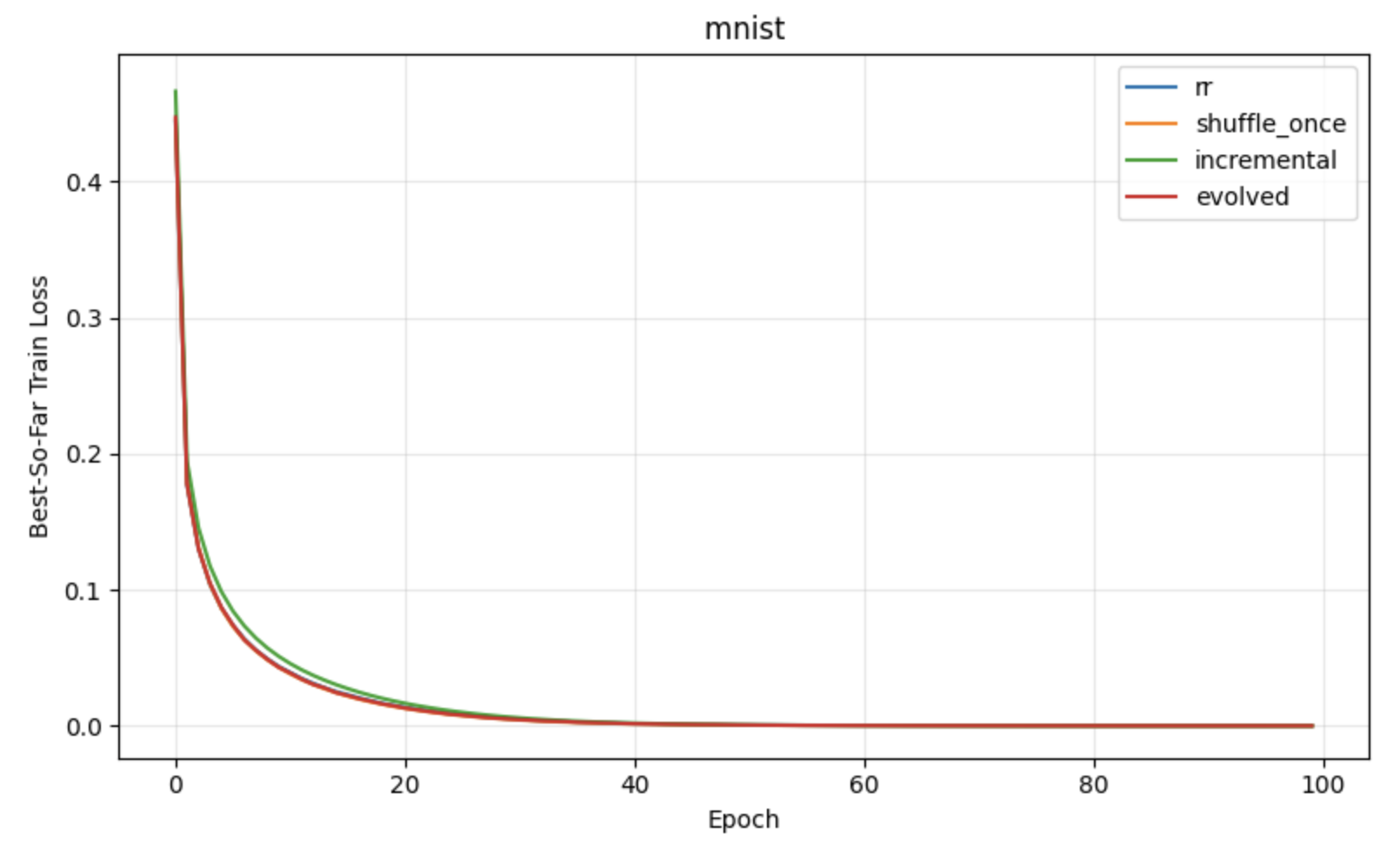}
        \includegraphics[width= 0.45\textwidth]{Figs/fashionmnist_poly_large.png}
        \caption{Classification Datasets (NN) - Diminishing Learning Rates}
        \label{fig_classification_nn_diminishing_lr}
    \end{subfigure}

    \caption{Classification Datasets (NN)}
    \label{fig:experiments_classification_NN_lr_large}
\end{figure}

\subsection{Small Learning Rates}

In the small-variance regime induced by very small learning rates, all shuffling schemes exhibit nearly indistinguishable behavior. 

\begin{figure}[!htbp]
    \centering

    \begin{subfigure}[h]{\textwidth}
        \centering
        \includegraphics[width= 0.32\textwidth]{Figs/a9a_constant_small.png}
        \includegraphics[width= 0.32\textwidth]{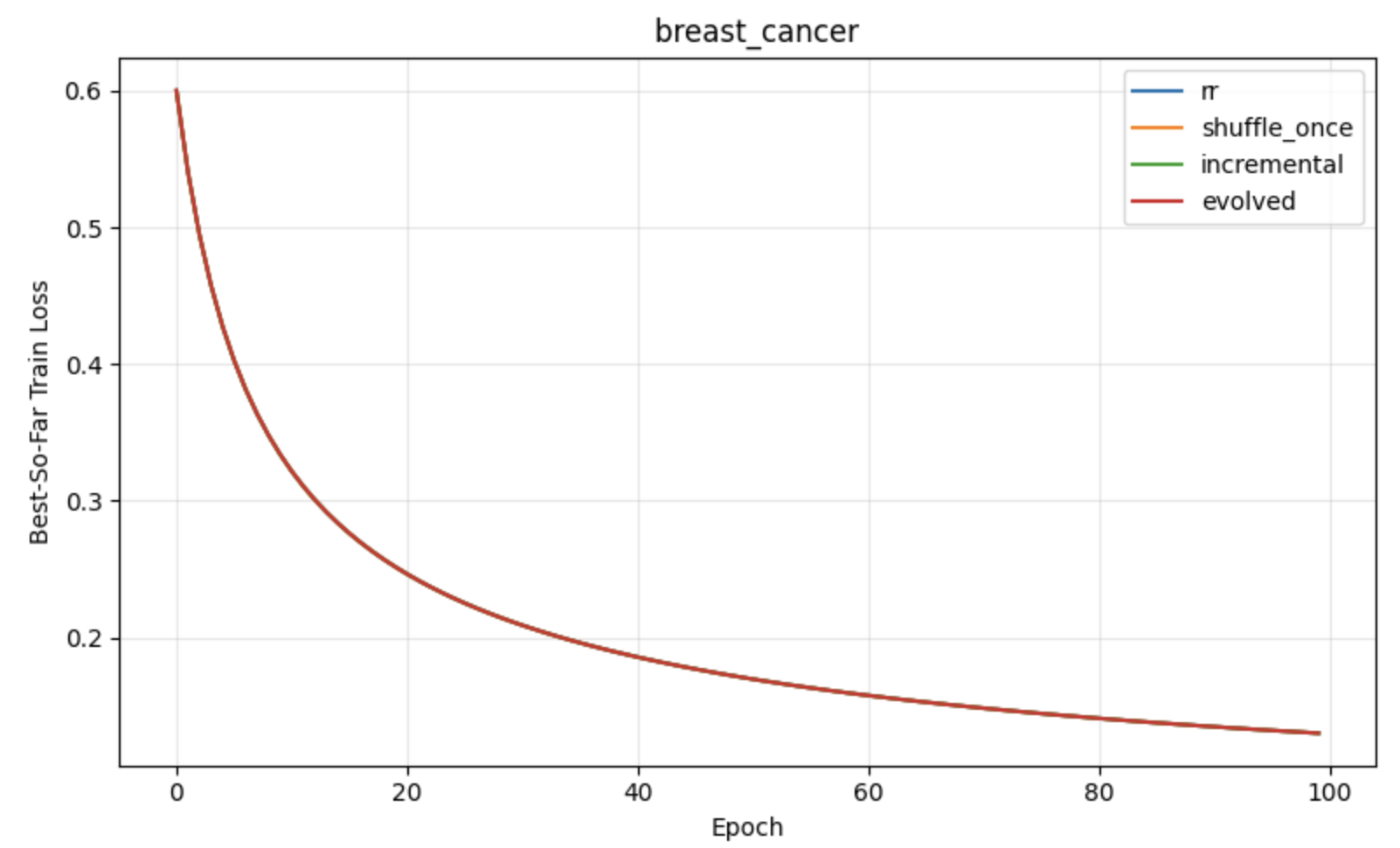}
        \includegraphics[width= 0.32\textwidth]{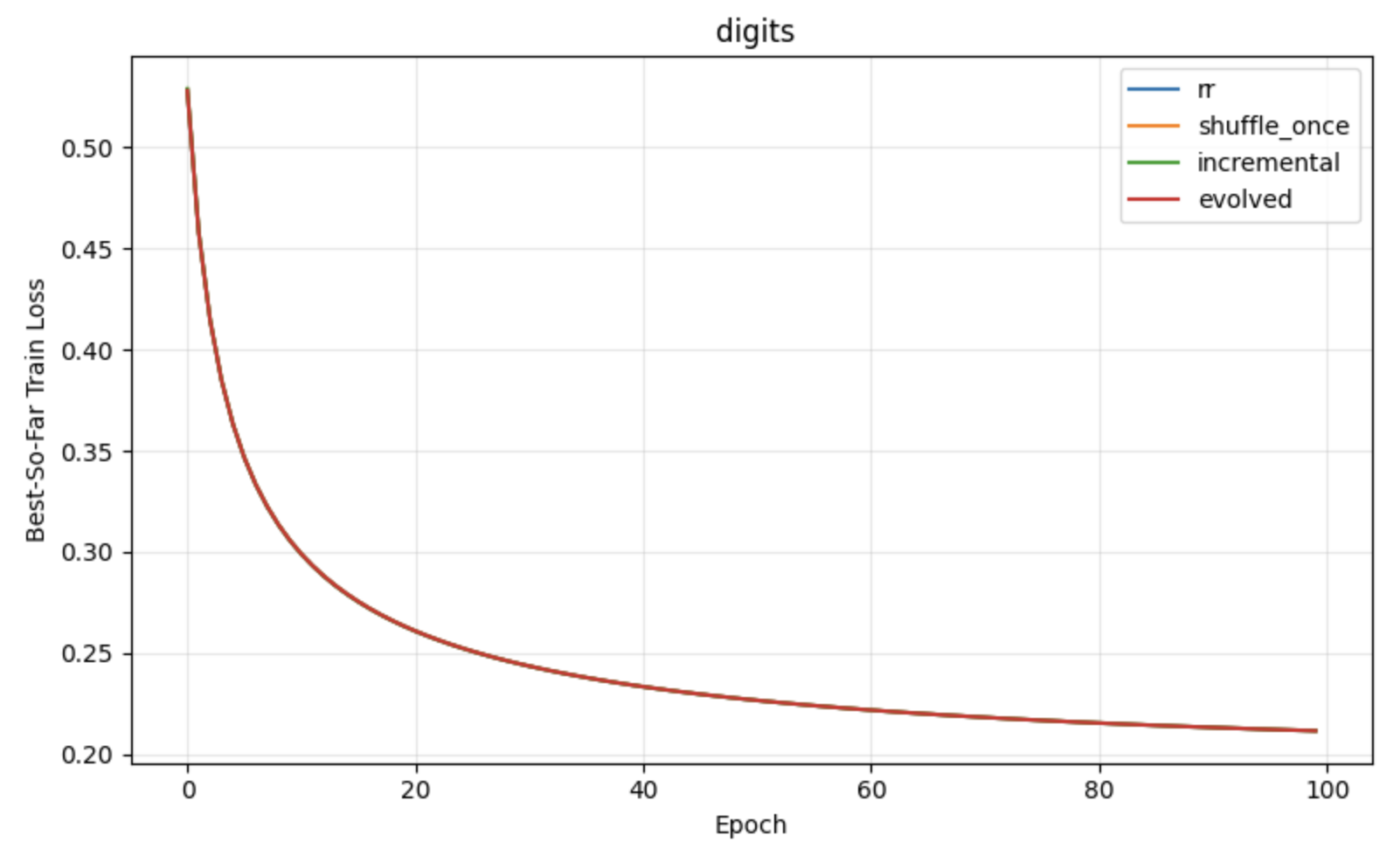}
        \caption{Classification Datasets - Constant Learning Rates}
        \label{fig_classification_constant_lr}
    \end{subfigure}

    \vspace{0.5em}

    \begin{subfigure}[h]{\textwidth}
        \centering
        \includegraphics[width= 0.32\textwidth]{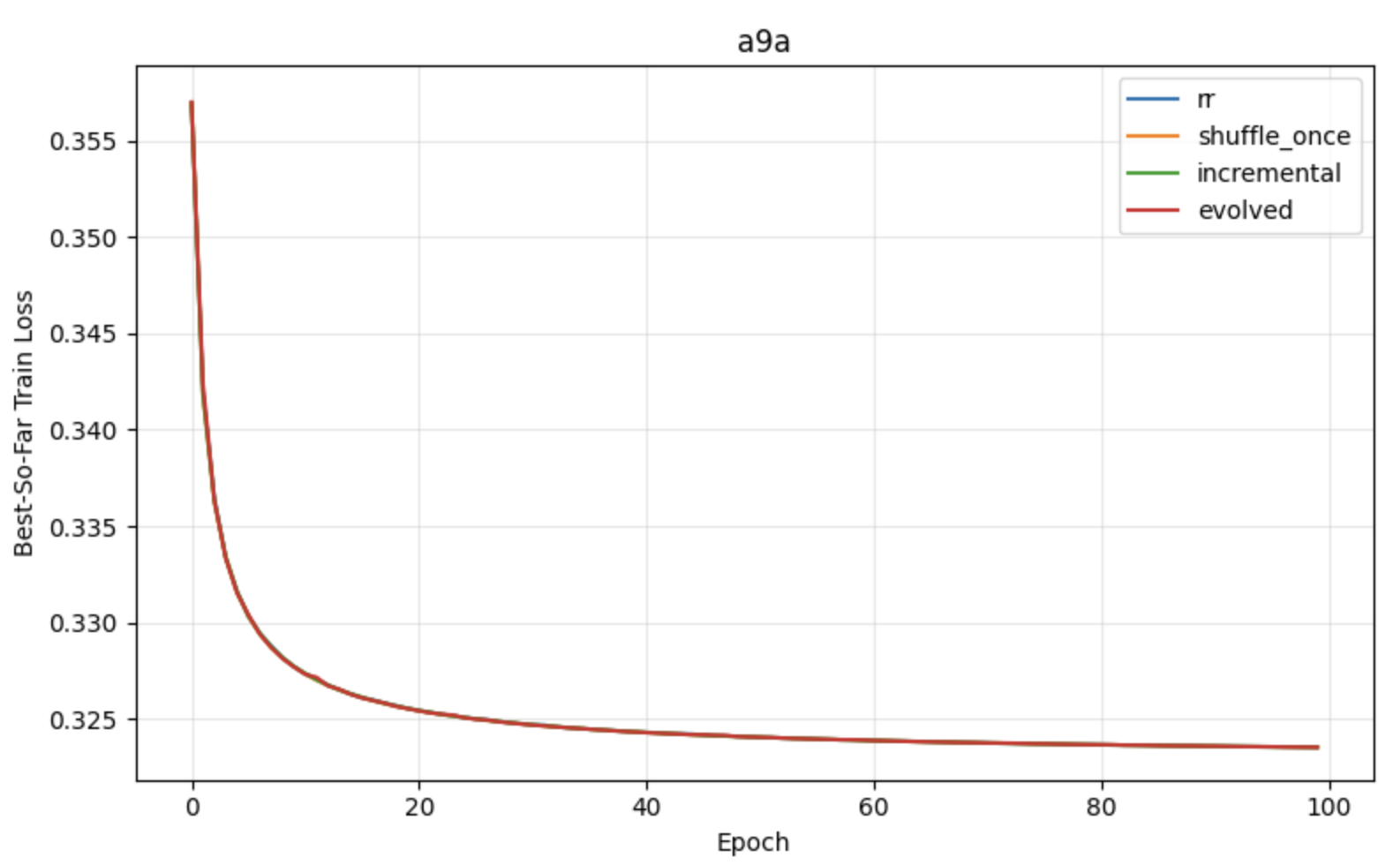}
        \includegraphics[width= 0.32\textwidth]{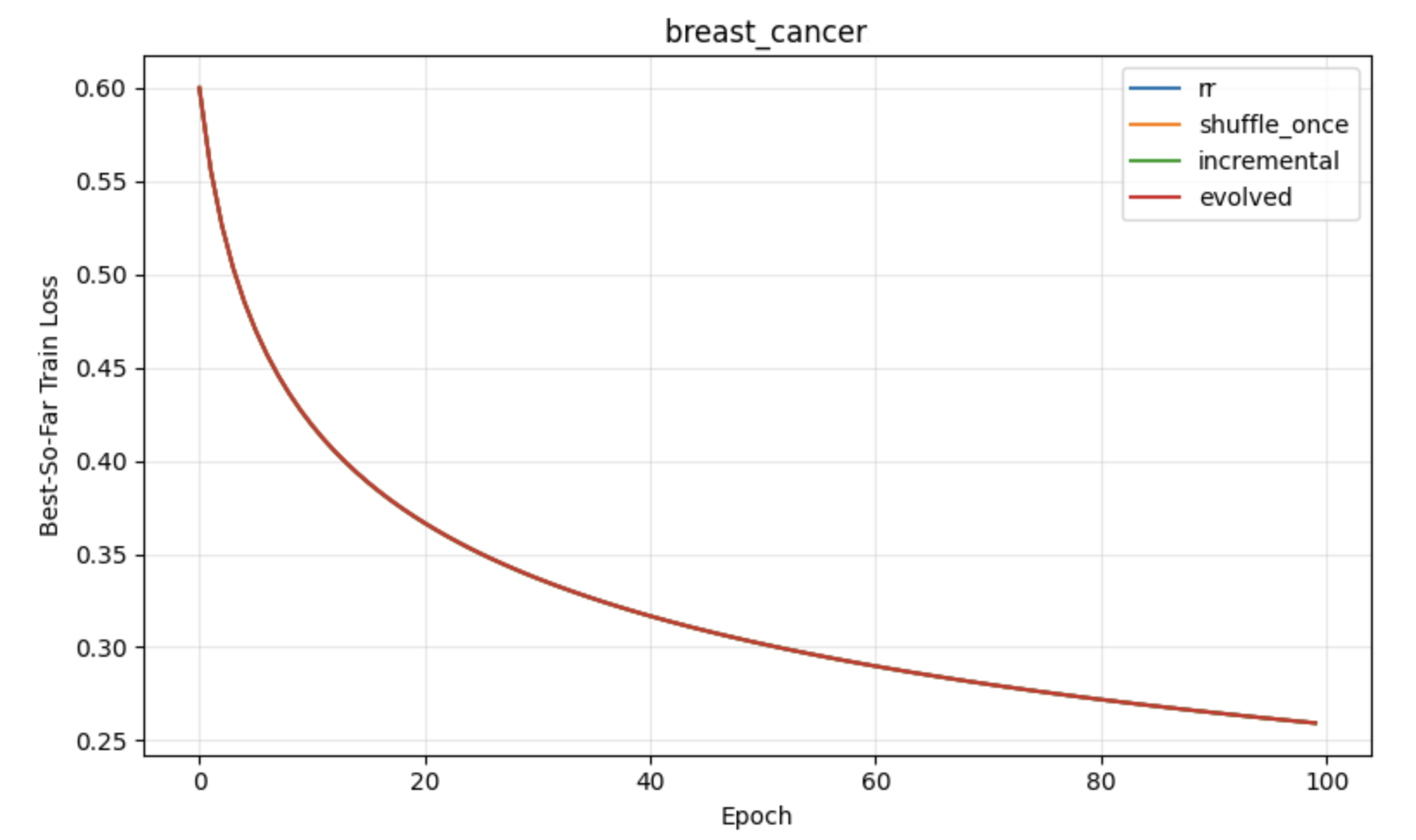}
        \includegraphics[width= 0.32\textwidth]{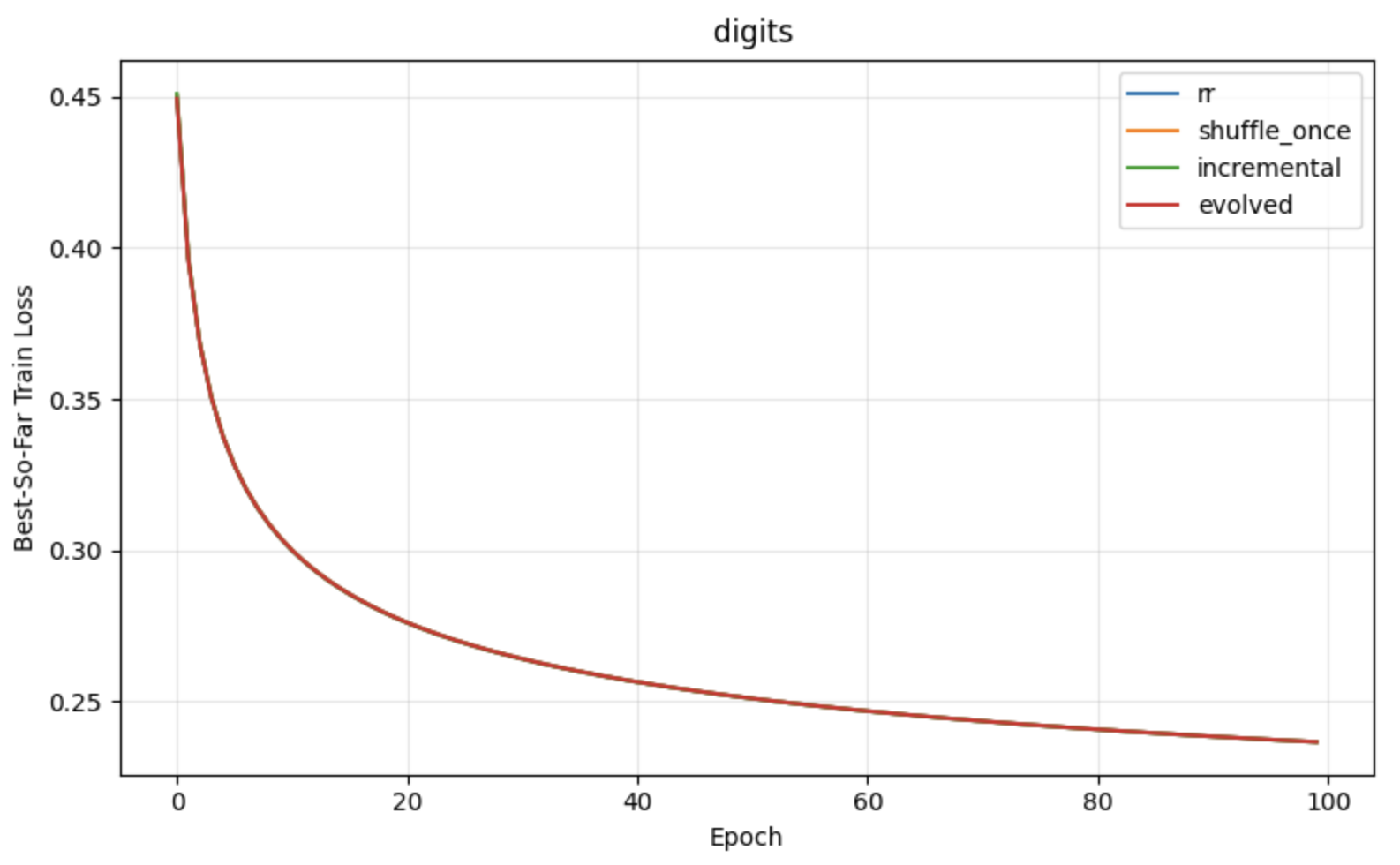}
        \caption{Classification Datasets - Diminishing Learning Rates}
        \label{fig_classification_diminishing_lr}
    \end{subfigure}

    \caption{Classification Datasets}
    \label{fig:experiments_classification_lr_small}
\end{figure}

\begin{figure}[!htbp]
    \centering

    \begin{subfigure}[h]{\textwidth}
        \centering
        \includegraphics[width= 0.32\textwidth]{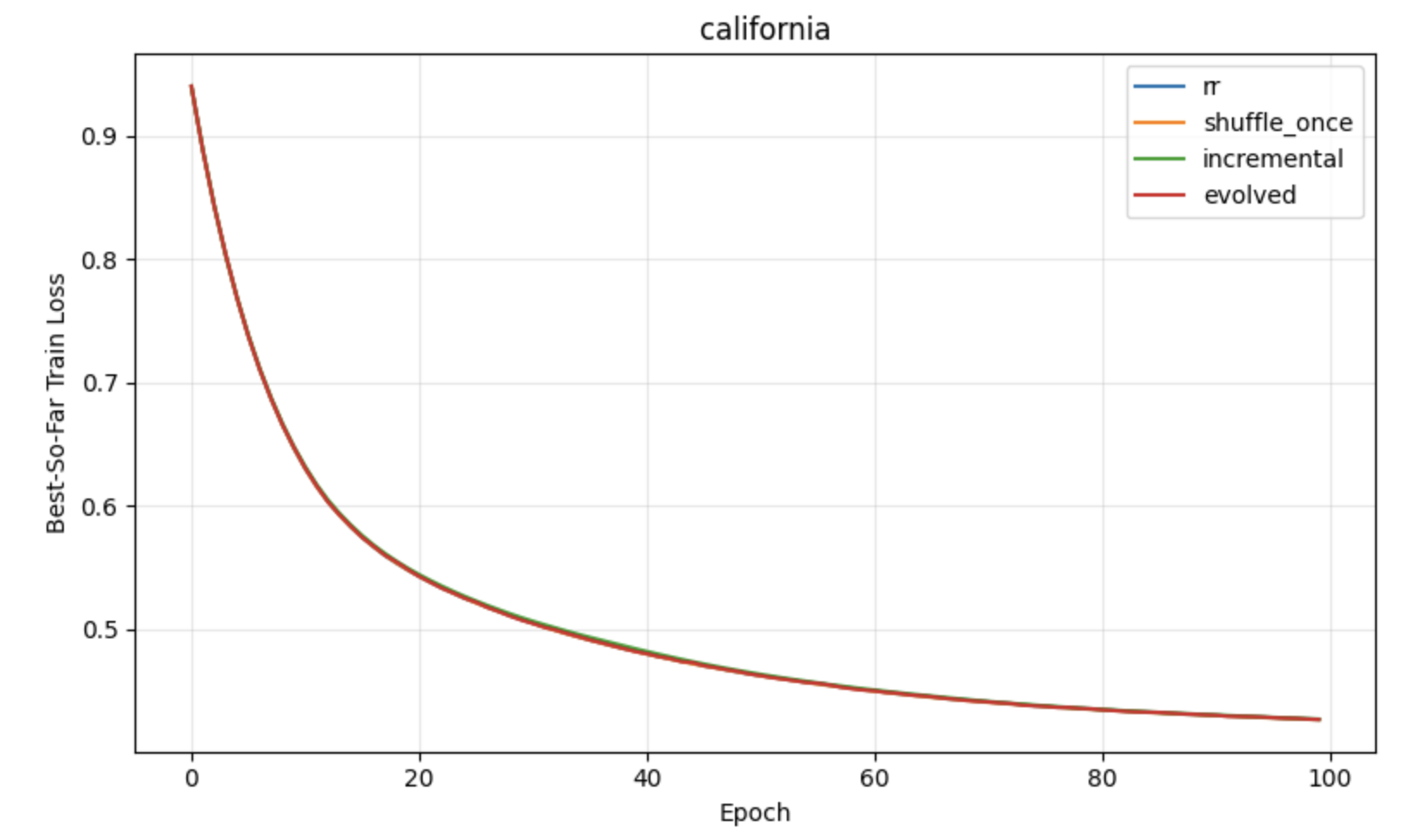}
        \includegraphics[width= 0.32\textwidth]{Figs/boston_constant_small.png}
        \includegraphics[width= 0.32\textwidth]{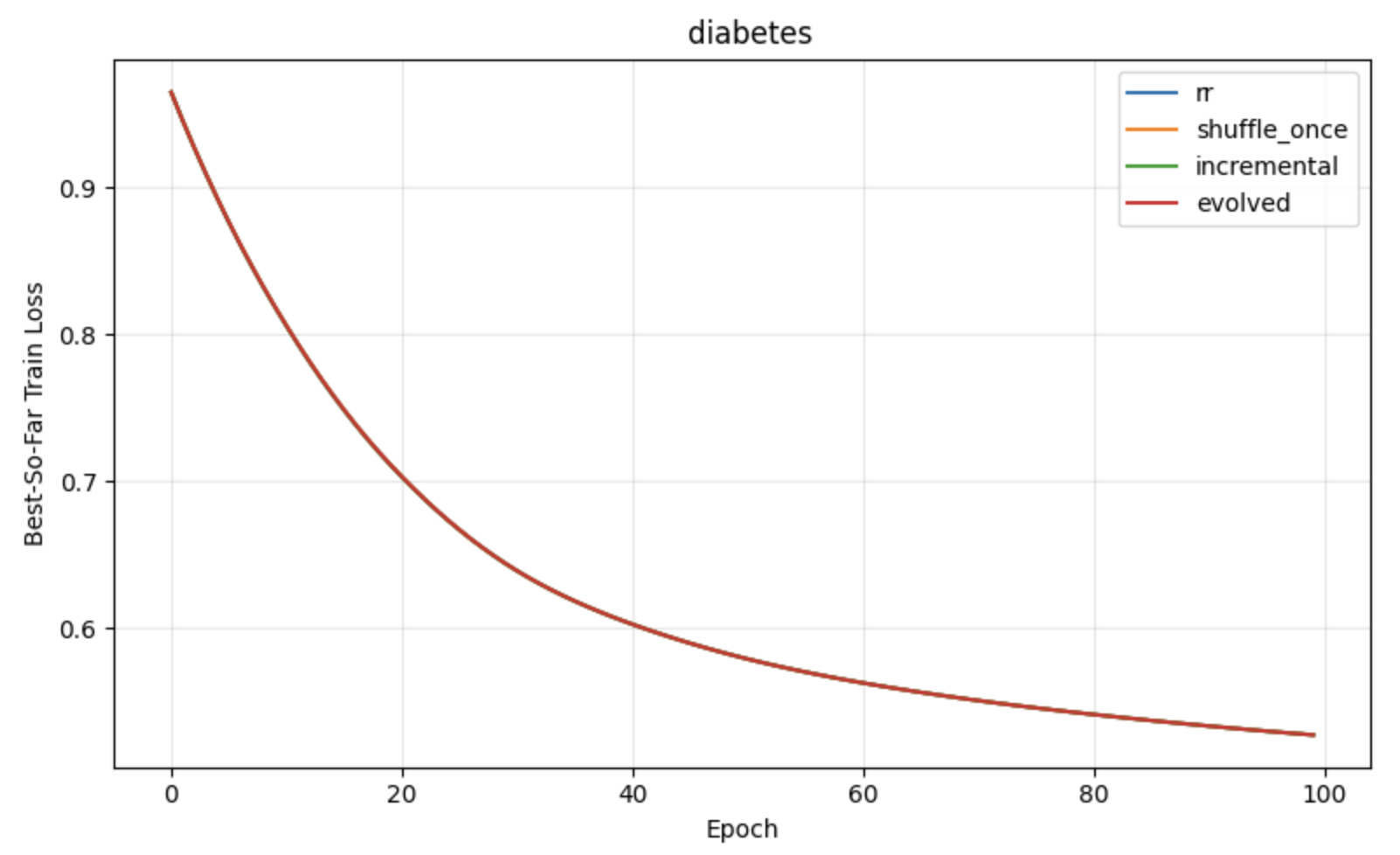}
        \caption{Regression Datasets - Constant Learning Rates}
        \label{fig_regression_constant_lr_small}
    \end{subfigure}

    \vspace{0.5em}

    \begin{subfigure}[h]{\textwidth}
        \centering
        \includegraphics[width= 0.32\textwidth]{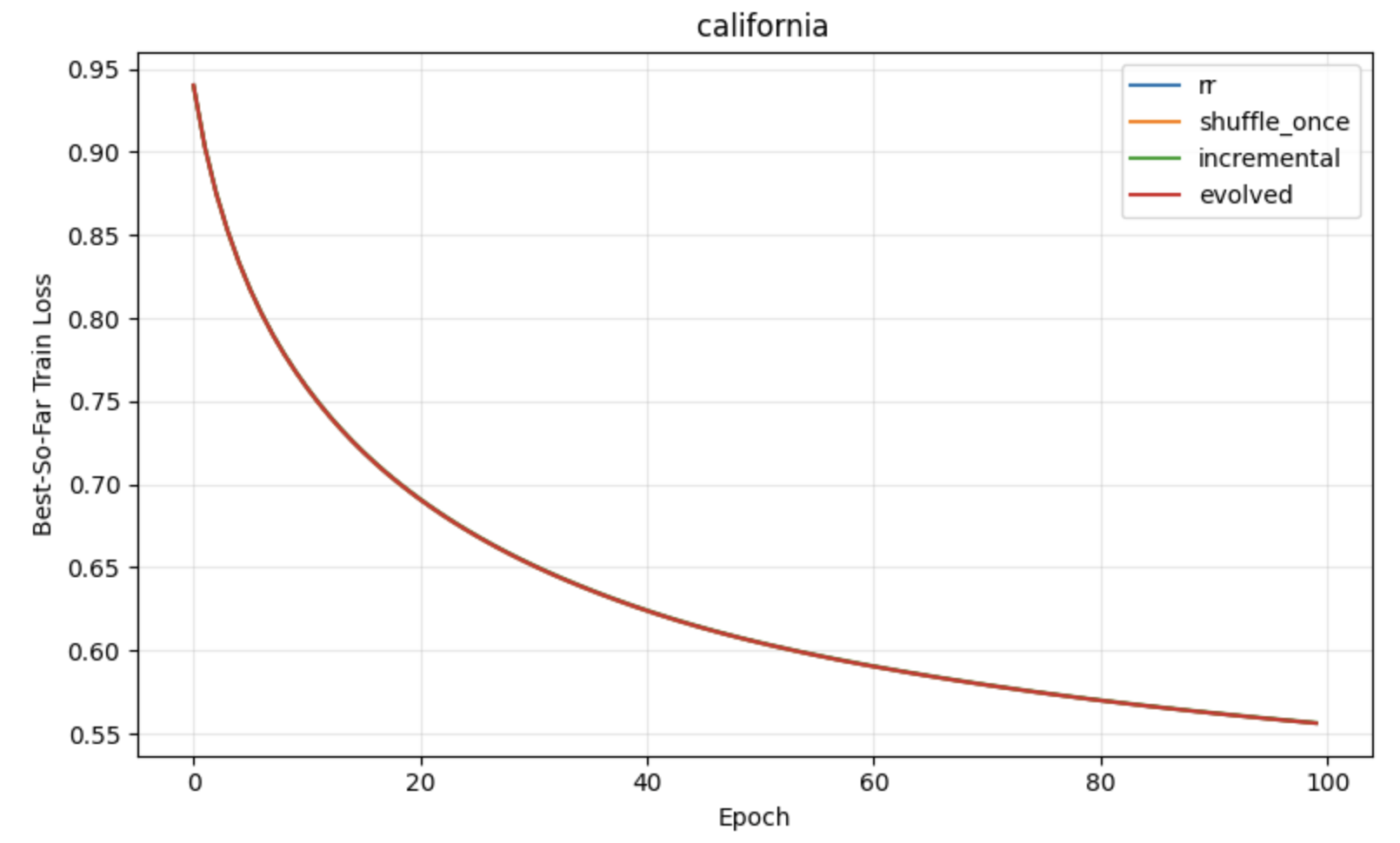}
        \includegraphics[width= 0.32\textwidth]{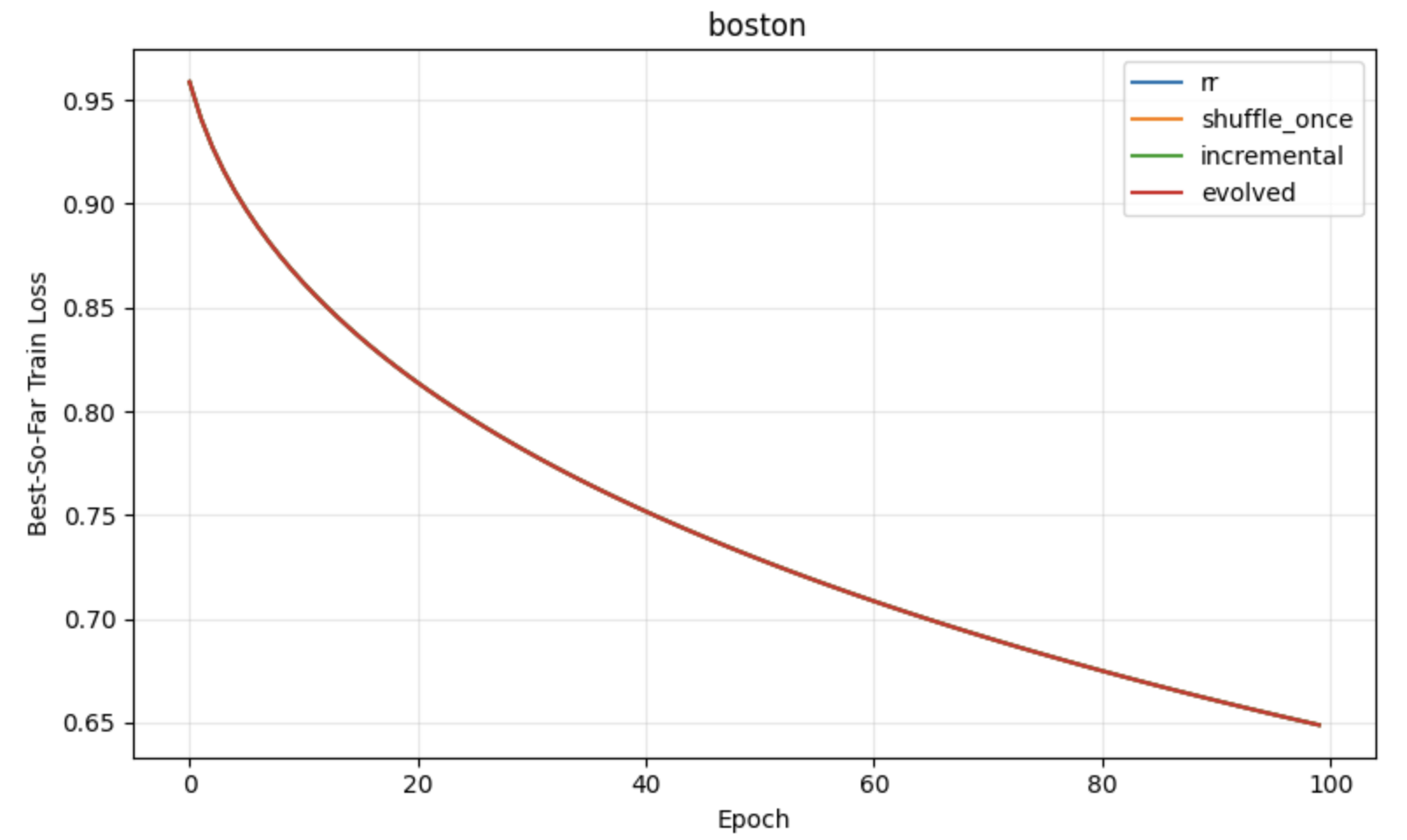}
        \includegraphics[width= 0.32\textwidth]{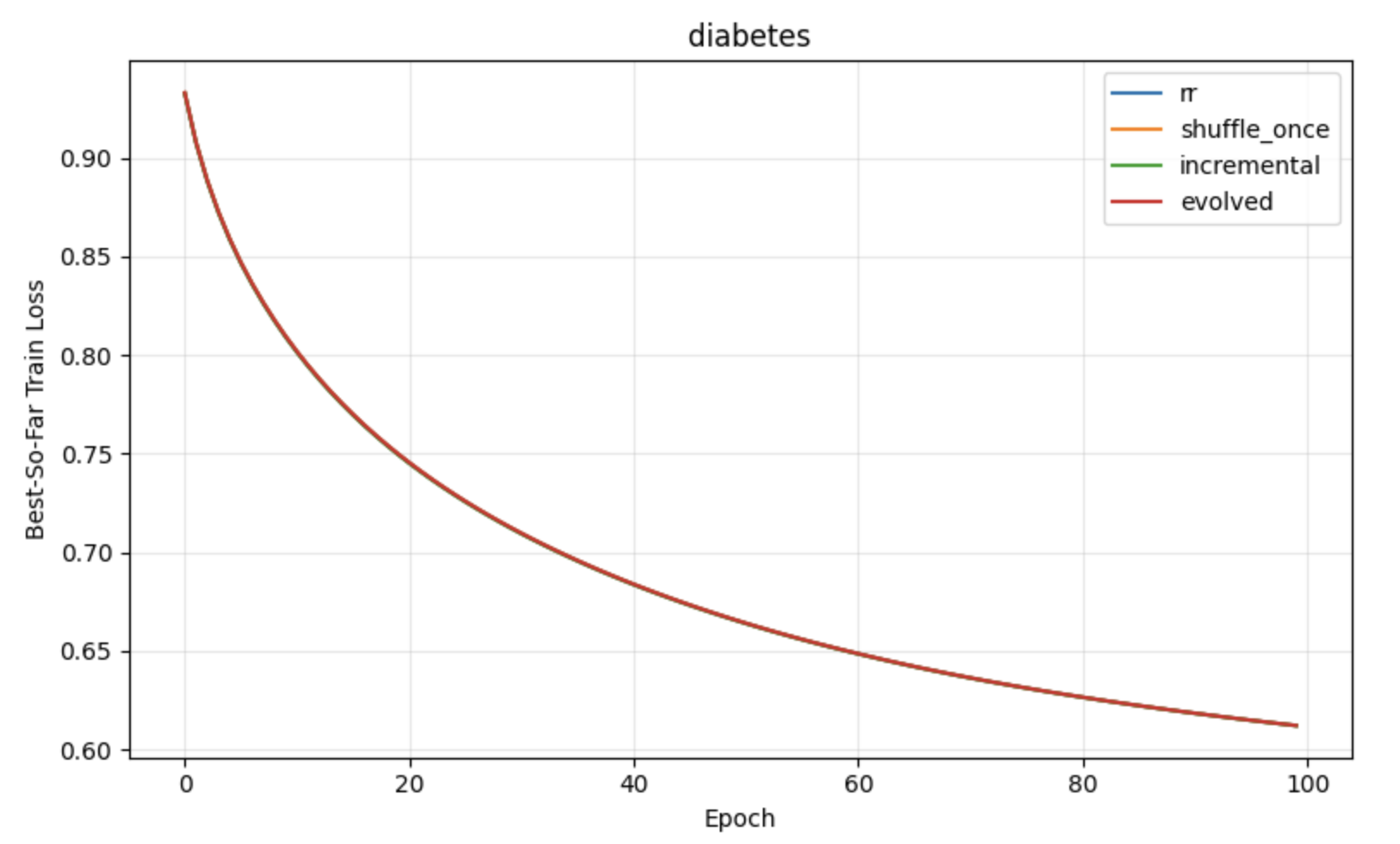}
        \caption{Regression Datasets - Diminishing Learning Rates}
        \label{fig_regression_diminishing_lr_small}
    \end{subfigure}

    \caption{Regression Datasets}
    \label{fig:experiments_regression_lr_small}
\end{figure}

\begin{figure}[!htbp]
    \centering

    \begin{subfigure}[h]{\textwidth}
        \centering
        \includegraphics[width= 0.45\textwidth]{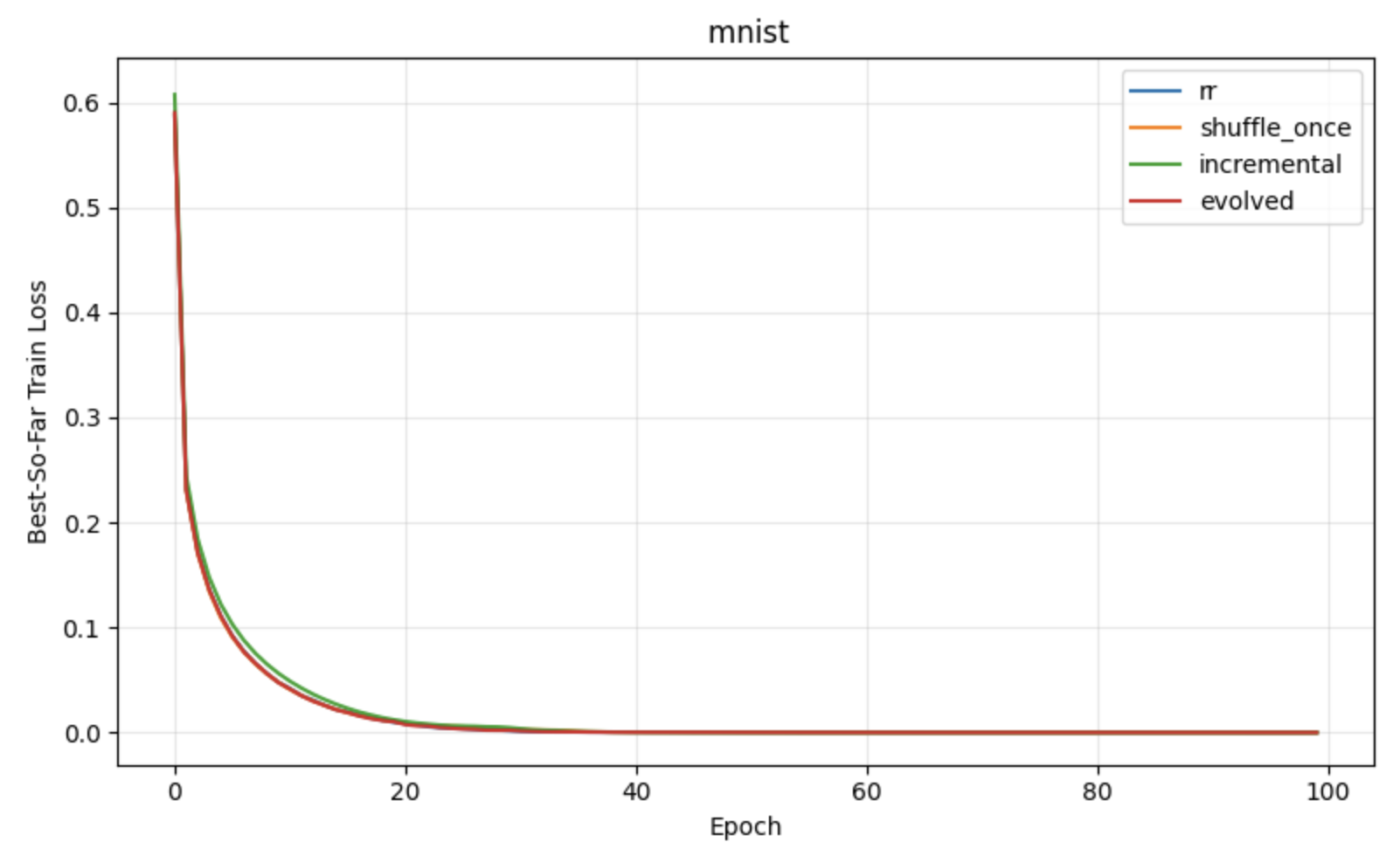}
        \includegraphics[width= 0.45\textwidth]{Figs/fashionmnist_constant_small.png}
        \caption{Classification Datasets (NN) - Constant Learning Rates}
        \label{fig_classification_nn_constant_lr_small}
    \end{subfigure}

    \vspace{0.5em}

    \begin{subfigure}[h]{\textwidth}
        \centering
        \includegraphics[width= 0.45\textwidth]{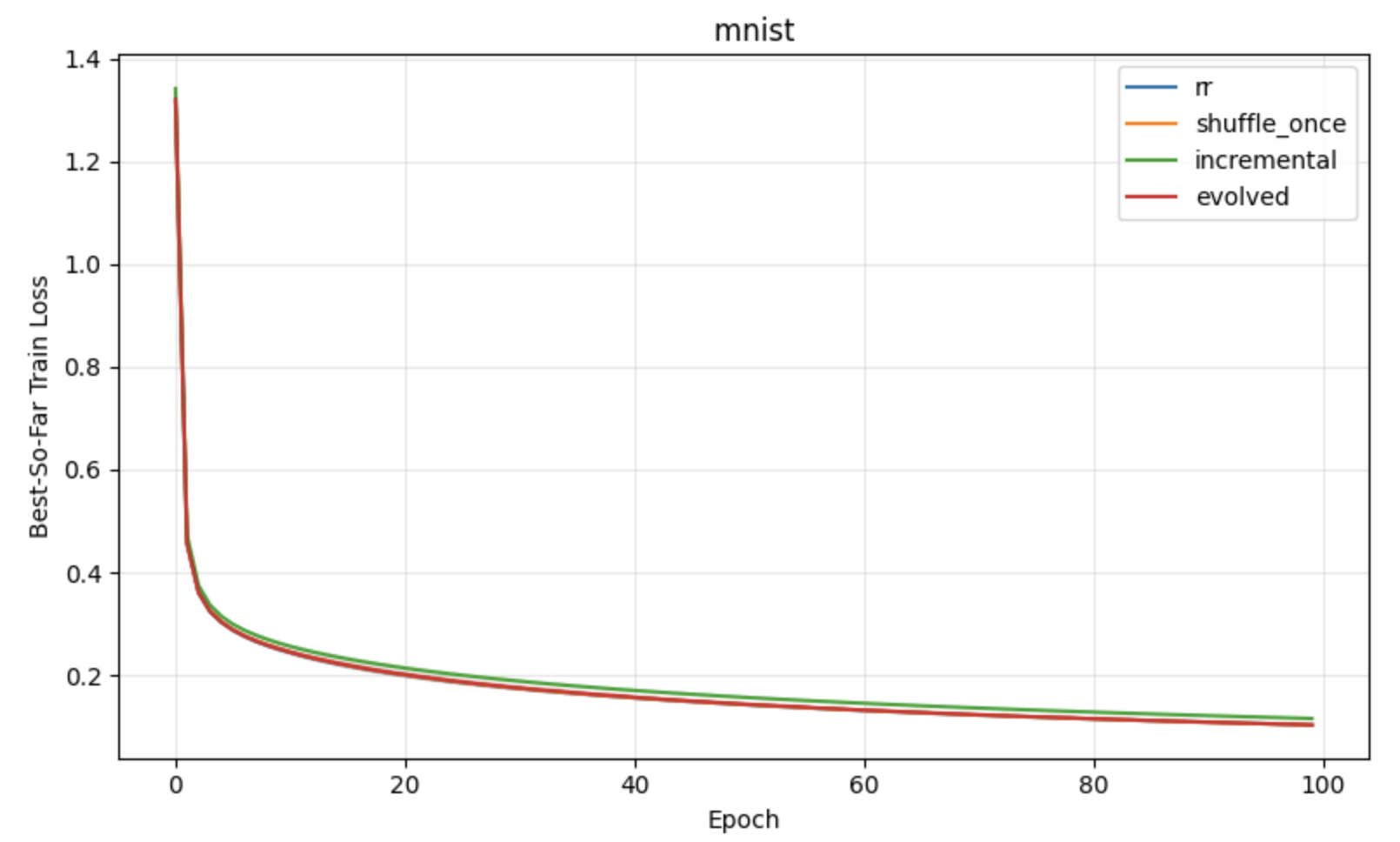}
        \includegraphics[width= 0.45\textwidth]{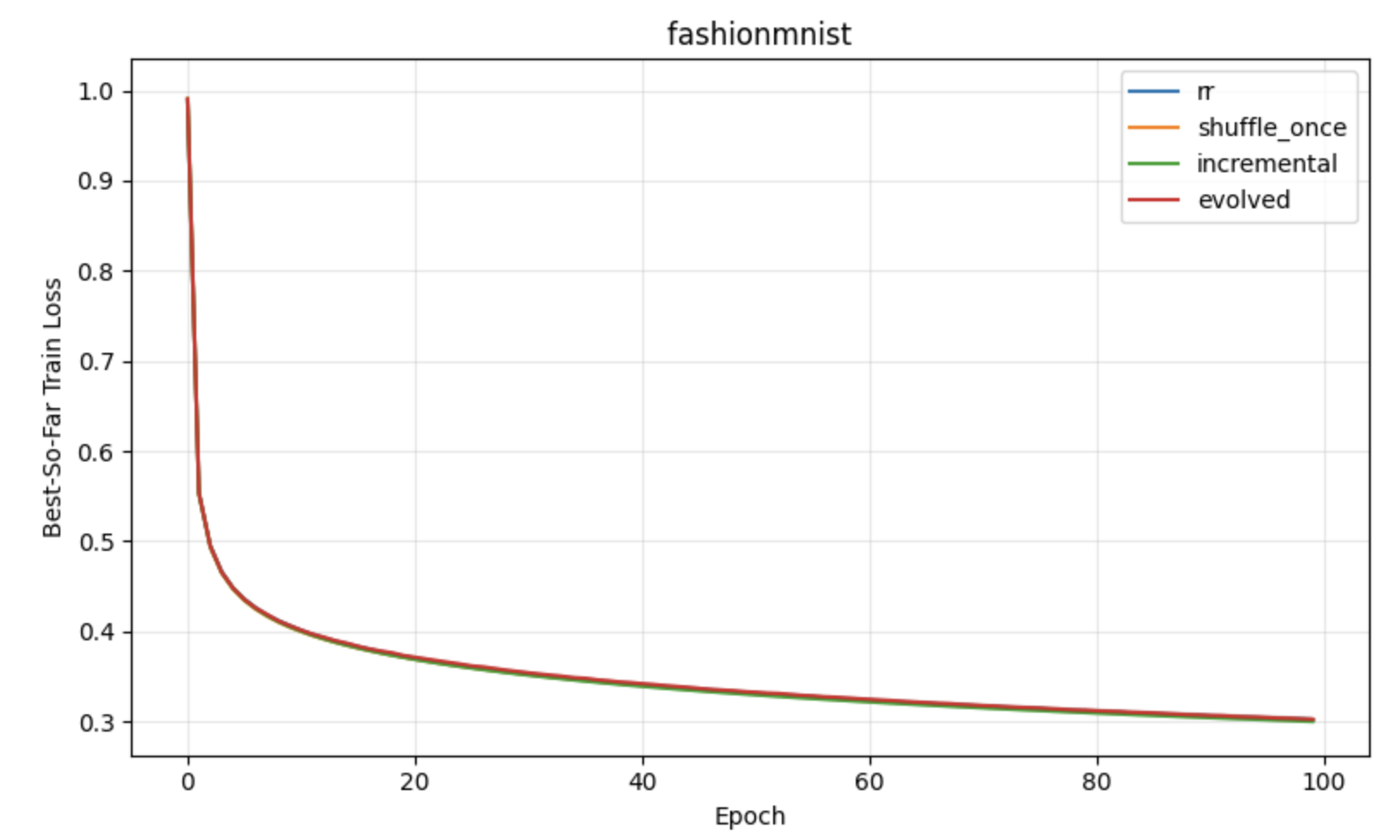}
        \caption{Classification Datasets (NN) - Diminishing Learning Rates}
        \label{fig_classification_nn_diminishing_lr_small}
    \end{subfigure}

    \caption{Classification Datasets (NN)}
    \label{fig:experiments_classification_nn_lr_small}
\end{figure}

\section{LLM-Guided Discovery of Shuffling Rules}\label{sec_app_LLM_guided}

We employ OpenEvolve \citep{openevolve}, a large language model (LLM)-guided program evolution pipeline based on the Microsoft Phi-4 model \citep{abdin2024phi}, to explore a space of deterministic shuffling rules.
Each candidate rule maps epoch-level training statistics, such as loss values or gradient norms from previous epochs, to a permutation of the dataset indices. To ensure practical feasibility, the search space is restricted to permutation generators with $O(n)$ computational complexity.

The pipeline iteratively generates candidate programs via LLM-based mutations, evaluates them on a suite of benchmark tasks, and retains high-performing rules for further refinement.
This process yields a concrete shuffling algorithm, which we refer to as Adaptive Block Reshuffling with Periodic Transforms (APR).
APR adaptively switches between structured blockwise permutations and reversal-based transforms depending on observed training progress.

\paragraph{Evolution setup.}
We used OpenEvolve, an LLM-guided program evolution framework, to explore adaptive data shuffling strategies for SGD.
The search space was restricted to deterministic Python functions producing permutations of $\{0,\dots,n-1\}$ in $O(n)$ time.
The evolved function could access only coarse epoch-level statistics (training loss and gradient norm) and was prohibited from modifying the model, optimizer, or learning rate.

\paragraph{Prompt and constraints.}
The LLM was instructed to design shuffling rules that reduce training loss and gradient norm as rapidly as possible.

\paragraph{Emergent structure.}
The evolved programs exhibited two components:
\begin{itemize}
\item Grouping indices into contiguous blocks and reshuffling at the block level.
\item Applying deterministic reversal or palindromic traversal of blocks, often periodically across epochs.
\end{itemize}

\paragraph{Role of the LLM.}
We emphasize that the LLM does not participate in training, inference, or optimization.
It is used only offline to explore a constrained design space and suggest candidate
shuffling rules.
All theoretical guarantees and experimental results apply to the final APR algorithm,
which is fully specified and deterministic.

\section{Settings for Adaptive Block Reshuffling with Periodic Transforms (APR)}\label{sec_app_experiment_settings}

\subsection{Explanation of Algorithm~\ref{alg:apr}}

\paragraph{Adaptive reshuffling with loss-ratio gating.}
Algorithm~1 generates a permutation $\pi_e \in \mathcal{S}_n$ at each epoch $e$ based on the observed training progress.
The algorithm uses an epoch-level loss ratio
\[
\rho_e = \frac{\ell_e}{\ell_{e-1}+\varepsilon},
\]
which measures relative improvement between consecutive epochs and serves as a coarse indicator of the optimization regime.

When $\rho_e$ indicates strong improvement ($\rho_e < \tau_{\mathrm{strong}}$), the algorithm applies \emph{block reshuffling} with small blocks of size $b_{\mathrm{strong}} = \max(1,\lfloor \alpha_{\mathrm{strong}} n \rfloor)$.
The dataset indices are partitioned into consecutive blocks, the blocks are randomly permuted, and the indices are processed block by block while preserving the order within each block.
At specified epochs, a reversal transform is additionally applied to mitigate order-dependent effects.

When improvement is mild ($\tau_{\mathrm{strong}} \le \rho_e < \tau_{\mathrm{mild}}$), the same block reshuffling procedure is used with larger blocks of size $b_{\mathrm{mild}} = \max(1,\lfloor \alpha_{\mathrm{mild}} n \rfloor)$, introducing more randomness while retaining partial structure.
When no improvement is detected ($\rho_e \ge \tau_{\mathrm{mild}}$), the algorithm falls back to full random reshuffling, optionally followed by an even-odd interleaving transform to reduce local correlations.

All randomization is controlled by a deterministic epoch-dependent seed, ensuring reproducibility.
Importantly, Algorithm~\ref{alg:apr} always outputs a single valid permutation per epoch and therefore remains within the standard without replacement SGD framework.
This allows the method to be analyzed using unified shuffling theory while enabling structured modifications of the data order in response to training dynamics.

\paragraph{Sequence transforms.}
Algorithm~\ref{alg:apr} optionally applies simple deterministic transforms to the generated permutation.
The reversal operator $\mathrm{Rev}(\cdot)$ reverses the order of indices within an epoch, i.e.,
$\mathrm{Rev}(x_1,\dots,x_n) = (x_n,\dots,x_1)$.
This transform is used to counteract systematic order effects arising from sequential data processing.

The even-odd interleaving operator $\mathrm{EO}(\cdot)$ reorders a sequence by listing all odd-indexed elements first, followed by all even-indexed elements,
$\mathrm{EO}(x) = (x_1,x_3,\dots;\,x_2,x_4,\dots)$.
This operation increases the typical distance between originally adjacent indices and helps reduce correlations induced by data locality.
Both operators preserve the without replacement structure and introduce no additional randomness or computational overhead.

\subsection{Experimental Setup for APR}

All experiments are conducted using a specific instance of Adaptive Block Reshuffling with Periodic Transforms (APR) discovered by an LLM-guided program evolution framework.
The LLM proposes concrete reshuffling rules with fixed parameter choices based on training feedback, and the instance reported here was selected due to its consistent empirical performance across datasets.
Importantly, this instance always produces a single permutation per epoch and therefore remains within the standard without-replacement SGD framework.

\paragraph{APR instance and parameter settings.}
The LLM-discovered reshuffling rule corresponds to Algorithm~\ref{alg:apr} with the following fixed parameters:
\begin{itemize}
    \item \textbf{Loss-ratio thresholds:}
    $\tau_{\mathrm{strong}} = 0.9$ and $\tau_{\mathrm{mild}} = 1.0$.
    \item \textbf{Block fractions:}
    $\alpha_{\mathrm{strong}} = 0.1$ (small blocks) and
    $\alpha_{\mathrm{mild}} = 0.2$ (medium blocks).
    \item \textbf{Reversal schedule:}
    reversal is applied every $p_{\mathrm{rev}} = 3$ epochs with phase $r_{\mathrm{rev}} = 0$.
    \item \textbf{Even--odd interleaving:}
    applied every $p_{\mathrm{eo}} = 3$ epochs with phase $r_{\mathrm{eo}} = 1$.
    \item \textbf{Initialization:}
    the first epoch uses a fully random permutation.
\end{itemize}

At each epoch $e$, the reshuffling rule computes the loss ratio
\[
\rho_e = \frac{\ell_e}{\ell_{e-1} + \varepsilon},
\]
with $\varepsilon = 10^{-10}$ for numerical stability.
When $\rho_e < \tau_{\mathrm{strong}}$, the algorithm applies block reshuffling with small blocks, followed by periodic reversal.
When $\tau_{\mathrm{strong}} \le \rho_e < \tau_{\mathrm{mild}}$, block reshuffling with larger blocks is used.
When $\rho_e \ge \tau_{\mathrm{mild}}$, a full random reshuffling is performed, optionally followed by even-odd interleaving.


\bibliographystyle{plainnat}
\bibliography{reference}





\end{document}